\documentclass[12pt]{article}

\usepackage{amsmath}
\usepackage{graphicx,psfrag,epsf}
\usepackage{enumerate}
\usepackage{natbib}
\usepackage{url} % not crucial - just used below for the URL 

\usepackage{subfigure}

\usepackage{comment}
\usepackage{rotating}
\usepackage{amsthm}
\usepackage{amssymb}
\usepackage{setspace}
\usepackage{color}

\makeatletter
\newcommand{\vast}{\bBigg@{4}}
\newcommand{\Vast}{\bBigg@{5}}
\makeatother

\newcommand{\bX}{\mathbf{X}}
\newcommand{\bx}{\mathbf{x}}

\newcommand{\bZ}{\mathbf{Z}}

\usepackage[a4paper]{geometry}

\newtheorem{thm}{Theorem}[section]
\newtheorem{lem}[thm]{Lemma}
\newtheorem{cor}[thm]{Corollary}
\newtheorem{rmrk}{Remark}

\begin{document}

\iffalse
\begin{center}
  \textbf{Receiver Operating Characteristic Curves and Confidence Bands for Support Vector Machines} \\
	\vspace{0.1in}
  \textbf{ Daniel J.\ Luckett$^{1}$, Eric B.\ Laber$^{2}$, Michael R.\ Kosorok$^{1}$}
  \\
	\vspace{0.1in}
$^1$ Department of Biostatistics, University of North Carolina, Chapel
  Hill, NC 27599 \\ 
$^2$ Department of Statistics, North Carolina State University,
  Raleigh, NC 27695
\end{center}
\fi

\begin{center}
  \textbf{Receiver Operating Characteristic Curves and Confidence Bands for Support Vector Machines} \\
	\vspace{0.1in}
	\textbf{Daniel J.\ Luckett$^{1}$,
	Eric B.\ Laber$^{2}$, Samer S.\ El-Kamary$^{3}$, Cheng Fan$^{5}$, Ravi Jhaveri$^{4}$, 
 Charles M.\ Perou$^{5}$, Fatma M.\ Shebl$^{6}$, and Michael R.\ Kosorok$^{1}$} \\
	$^{1}$Department of Biostatistics, University of North Carolina, Chapel Hill, North Carolina, U.S.A. \\
	$^{2}$Department of Statistics, North Carolina State University, Raleigh, North Carolina, U.S.A. \\
	$^{3}$Department of Epidemiology and Public Health, University of Maryland, Baltimore, Maryland, U.S.A. \\
	$^{4}$Department of Pediatrics, University of North Carolina, Chapel Hill, North Carolina, U.S.A. \\
	$^{5}$Department of Genetics, University of North Carolina, Chapel Hill, North Carolina, U.S.A. \\
	$^{6}$Department of Epidemiology, Yale University, New Haven, Connecticut, U.S.A.
\end{center}

\normalsize

\begin{abstract}
  Many problems that appear in biomedical decision making, 
	such as diagnosing disease and predicting response 
	to treatment, can be expressed as binary classification problems. 
	The costs of false positives and false negatives vary across application domains and  
	receiver operating characteristic (ROC) curves provide a visual 
	representation of this trade-off. Nonparametric estimators for the ROC curve, 
	such as a weighted support vector machine (SVM), are desirable because they are robust to 
	model misspecification. While weighted SVMs have great potential for estimating 
	ROC curves, their theoretical properties were heretofore underdeveloped. 
	We propose a method for constructing confidence bands 
	for the SVM ROC curve and provide the theoretical 
	justification for the SVM ROC curve by showing that the risk 
	function of the estimated decision rule is uniformly consistent 
	across the weight parameter. 
	We demonstrate the proposed confidence band method 
	and the superior sensitivity and specificity of the weighted SVM compared to 
	commonly used methods in diagnostic medicine using simulation studies. 
	We present two illustrative examples: diagnosis of hepatitis C and a 
	predictive model for treatment response in breast cancer.
\end{abstract}

\noindent%

%\newpage

\section{Introduction}

Many important problems in biomedical decision making can be 
expressed as binary classification problems. 
For example, one may wish to identify infants 
infected with hepatitis C virus from a sample of infants born
to infected mothers \citep[][]{shebl2009prospective}, 
%combine multiple biomarkers to improve screening for prostate cancer 
screen for prostate cancer using prostate-specific antigen
\citep[][]{etzioni1999incorporating}, 
or predict which breast cancer patients will respond to treatment 
based on genetic characteristics \citep[][]{fan2011building}. 
In many applications, the costs of false positives and false negatives 
may differ and classification methods must allow for 
unequal weighting of these errors. 
%classification can be improved by placing unequal 
%weights on false positives and false negatives. 
%modeling the data-generating mechanism is difficult and 
%classification methods need to be robust to model misspecification.
%We present an approach to estimating the optimal receiver operating 
%characteristic (ROC) curve using a weighted support vector machine (SVM) 
%and introduce a bootstrap method for constructing confidence bands 
%for the SVM ROC curve. 
We present an approach to estimating receiver operating characteristic 
(ROC) curves using a weighted support vector machine (SVM) and introduce a bootstrap 
method for constructing confidence bands for the SVM ROC curve. 

%The SVM of \cite{cortes1995support} is 
%a popular classification algorithm based on risk minimization 
%\citep[see also][]{krzyzak1996nonparametric, lin2002support, 
%zhang2004statistical, steinwart2008support}. 
%However, risk minimization does not immediately allow for 
%placing unequal weights on false positives and false negatives. 
%Unequal weighting is important in applications such as diagnostic 
%medicine, where misdiagnosing a diseased patient 
%as nondiseased often has different consequences from 
%misdiagnosing a nondiseased patient as diseased. 
%When applying classification methods in practice, it may be important 
%to place unequal weights on false positives and false negatives. 
%For example, misdiagnosing a diseased patient 
%as nondiseased often has different consequences from 
%misdiagnosing a nondiseased patient as diseased. 
%The trade-off between false positives 
%and false negatives is often displayed using an ROC curve 
Receiver operating characteristic curves plot the false positive fraction against the 
false negative fraction across values of the classification cut point
%Receiver operating characteristic curves provide a useful tool in classification 
\citep[][]{zhou2002statistical, pepe2003statistical}. 
Various methods for modeling and estimating ROC curves 
have been proposed, including parametric regression models 
\citep[][]{pepe1997regression, mcintosh2002combining} 
and semiparametric regression models 
\citep[][]{pepe2000interpretation, cai2002semiparametric, cai2008regression}.  
%\citep[][]{pepe1997regression, pepe2000interpretation, 
%cai2002semiparametric, mcintosh2002combining, cai2008regression, 
%chrzanowski2014weighted}. 
Existing methods for ROC curve confidence bands involve estimating the 
biomarker distributions in the diseased and nondiseased samples using 
parametric models \citep[][]{ma1993confidence} or kernel density estimators 
\citep[][]{jensen2000regional, claeskens2003empirical, horvath2008confidence}, 
or using empirical distribution functions in combination with the bootstrap 
\citep[][]{campbell1994advances}. Existing methods for ROC curve confidence bands assume a scalar biomarker. 
In the current setting, we apply the SVM of \cite{cortes1995support} 
to classification with a multivariate biomarker \citep[see also][]{krzyzak1996nonparametric, lin2002support, 
zhang2004statistical, steinwart2008support}.
The ROC curve is constructed by varying the weight placed on false positives and false 
negatives in the objective function rather than varying the classification cut point.  
Because the SVM classifier may vary across the range of the weight parameter, 
existing confidence band methods that assume a scalar biomarker cannot be directly applied. 

%A number of authors have developed methods to adapt machine 
%learning techniques to allow for unequal weighting of 
%false positives and false negatives. 
Machine learning techniques that output a continuous score or predicted probability 
allow for straightforward application of ROC curve methodology 
\citep[see, e.g.,][]{spackman1989signal, bradley1997use, provost1997analysis, provost1998robust}. 
However, there are fewer examples of applying ROC curve methodology to classifiers that 
output only a class label, such as the SVM. 
\cite{platt1999probabilistic} proposed a method to extract class probabilities from 
the output of the SVM \citep[see also][]{vapnik1998statistical, lin2007note}. 
However, these methods rely on fitting parametric models to the SVM class labels. 
Example 2.5 in \cite{steinwart2008support} 
discusses classification using a weighted SVM but does not 
explore tuning these weights to achieve desired operating characteristics 
(e.g., a specific false positive fraction). 
\cite{veropoulos1999controlling} propose using weights 
to control the sensitivity and specificity of the SVM and to estimate an ROC curve, 
but provide no theoretical results or inference methods. 
As such, the weighted SVM has not yet been extensively applied in practice. 
We build on the work of \cite{veropoulos1999controlling} 
by deriving theoretical properties and 
developing a bootstrap method for constructing confidence 
bands for the SVM ROC curve. 

There are numerous applications to motivate this work; however, we focus on two primary 
illustrative applications. The first is diagnostic testing 
for infant hepatitis C virus (HCV). Existing diagnostic tests exhibit poor sensitivity for predicting 
which infants will become chronically infected. A weighted SVM using multiple 
biomarkers is able to improve performance over standard HCV diagnostic tests. 
The second illustrative application we consider is 
predicting which breast cancer patients will respond to treatment. Genomic data provide a wealth of information 
for this purpose. However, it is difficult to specify a parametric model for response given 
genomic features because of the high dimension of genomic data. Because the SVM provides 
nonparametric classification \citep[][]{steinwart2008support}, it is a natural choice for this problem. 

In Section~\ref{roc.setup}, we present the method developed by 
\cite{veropoulos1999controlling} and introduce a method for bootstrap confidence bands. 
In Section~\ref{roc.theory}, we show a number of theoretical results, 
including that the risk of the estimated 
decision function is uniformly consistent 
across the weight parameter. 
In Section 4, we present a series of simulation experiments 
comparing the performance of the weighted SVM to 
standard methods in diagnostic medicine 
including logistic regression \citep[][]{mcintosh2002combining} and 
semiparametric ROC curves \citep[][]{cai2004semi}
and to evaluate the operating characteristics of the proposed bootstrap confidence bands. 
In Section~\ref{roc.data}, we present illustrative case studies. 
We conclude and discuss future research in Section~\ref{roc.conc}. 
Proofs and additional simulation results are provided in 
Appendix A and Appendix B.

\section{Weighted Support Vector Machines} \label{roc.setup}
 
\subsection{ROC Curve Estimation}

Assume that the available data are 
$(A_i, \bX_i)$, $i = 1,\ldots, n$, 
which comprise $n$ i.i.d.\ copies of $(A, \bX)$, where $A \in \{-1, 1\}$
is a class label (e.g., in diagnostic medicine, $A = 1$ corresponds to a 
diseased individual and $A = -1$ corresponds to a nondiseased
individual) and $\bX \in \mathcal{X} \subseteq \mathbb{R}^p$ are 
covariates. The goal is to estimate a classifier that
correctly identifies a patient's class label based on that
patient's covariates. Consider minimizing the expected weighted
misclassification, where each misclassification event is weighted by
the cost function
$C_a(\alpha) = \left\lbrace 1 + (2\alpha - 1) a \right\rbrace / 2 = \alpha 1(a = 1) +
(1 - \alpha)1(a = -1)$,
where $C_a(\alpha)$ is the cost of misclassification when the true class label is $A = a$. 
In diagnostic medicine, with $A = 1$ corresponding to disease and $A = -1$ corresponding to nondisease, 
$\alpha$ determines the relative weight placed on the
sensitivity and specificity of the test. When $\alpha = 1/2$, 
sensitivity and specificity are given equal weight and 
the cost function reduces to zero-one misclassification error.  
Let $\mathcal{D}$ denote a class of functions from $\mathcal{X}$ into $\{-1, 1\}$. Then, 
the optimal classifier with respect to cost function $C_a(\alpha)$ within $\mathcal{D}$ is 
\begin{equation} \label{D_tilde}
\widetilde{D}_\alpha = \operatorname*{arg\,min}_{D \in \mathcal{D}} \mathbb{E} \Big[ 1\{ D(\bX) \neq A \} C_A(\alpha) \Big]. 
\end{equation}

For fixed $\alpha \in (0, 1)$ and a classifier $D$, 
the plug-in estimator of the weighted misclassification error is 
%the observed data can be used to approximate the weighted misclassification in (\ref{D_tilde})  
%$n^{-1} \sum_{i=1}^n 1\{ D(\bX_i) \neq A_i \} C_{A_i}(\alpha)$. 
%as 
$\mathbb{E}_n 1\{ D(\bX) \neq A \} C_{A}(\alpha)$, 
where $\mathbb{E}_n$ is the empirical measure of the observed data. 
Note that any classifier $D(\bX)$ can be represented as 
$\mathrm{sign}\big\{f(\bX)\big\}$ for some decision function $f : \mathcal{X} \rightarrow \mathbb{R}$; 
we will assume that the decision function is smooth and thus 
$f$ belongs to a class of smooth functions, $\mathcal{F}$. 
For example, we can let $\mathcal{F}$ be the space of 
linear functions, the space of polynomial functions, 
or the reproducing kernel Hilbert space (RKHS) 
associated with the Gaussian kernel \citep[][]{steinwart2008support}. 
The weighted misclassification error associated with decision function $f$ 
is $\mathbb{E} \left[ 1\left\{A f(\bX) < 0 \right\} C_A(\alpha)\right]$. 
Minimizing the empirical risk is difficult due to the discontinuity of the 
indicator function. Using the hinge loss, $\phi(u) = \mathrm{max}(0, 1 - u)$, 
as a surrogate loss function \citep[][]{bartlett2006convexity}, 
an estimator for the optimal decision function is 
\begin{equation} \label{wsvm}
\widehat{f}^{\lambda_n}_{\alpha, n} = \operatorname*{arg\,min}_{f \in \mathcal{F}} 
\left[ \mathbb{E}_n \phi\{A f(\bX)\} C_{A}(\alpha) + \lambda_n \|f\|^2 \right], 
\end{equation}
where $\| \cdot \|$ is a norm on $\mathcal{F}$ and $\lambda_n$ is 
a penalty parameter. We discuss how to choose a value of $\lambda_n$ in Section~\ref{roc.simul}. 
In the following, we write $\widehat{f}_\alpha$ in place of $\widehat{f}^{\lambda_n}_{\alpha, n}$ 
to simplify notation. The problem of estimating the optimal 
classifier  in (\ref{wsvm}) can be solved using the SVM introduced by 
\cite{cortes1995support}. 

We estimate the optimal classifier, $\widetilde{D}_\alpha$, 
using $\widehat{D}_\alpha(\bX) = \mathrm{sign}\left\{\widehat{f}_\alpha(\bX)\right\}$. 
For any $\alpha \in (0, 1)$, we can estimate the sensitivity and specificity 
of the estimated classifier using the empirical estimators 
$\widehat{se}\left(\widehat{f}_\alpha\right) = \mathbb{E}_n 
1(A = 1) 1\left[\mathrm{sign}\left\{\widehat{D}_\alpha (\bX)\right\} = 1\right]
 / \mathbb{E}_n 1(A = 1)$ and 
$\widehat{sp}\left(\widehat{f}_\alpha\right) = \mathbb{E}_n 
1(A = -1) 1\left[\mathrm{sign}\left\{\widehat{D}_\alpha (\bX)\right\} = -1\right]
 / \mathbb{E}_n 1(A = -1)$. 
%$\widehat{se}\left(\widehat{f}_\alpha\right) = \mathbb{E}_n 
%1\left[ A = \mathrm{sign}\left\{\widehat{D}_\alpha (\bX)\right\} = 1\right]
% / \mathbb{E}_n 1(A = 1)$ and 
%$\widehat{sp}\left(\widehat{f}_\alpha\right) = \mathbb{E}_n 
%1\left[ A = \mathrm{sign}\left\{\widehat{D}_\alpha (\bX)\right\} = -1\right]
% / \mathbb{E}_n 1(A = -1)$. 
Plotting $\widehat{se}\left(\widehat{f}_\alpha\right)$ against 
$1 - \widehat{sp}\left(\widehat{f}_\alpha\right)$ 
as functions of $\alpha$  will yield a nonparametric estimator of the optimal ROC curve. 
The ROC curve encodes a continuum of classifiers indexed by $\alpha$; to select a single classifier, 
there are a number of methods for defining an optimal value, say $\alpha^*$, for $\alpha$.
For example, one could choose the $\alpha^*$ which leads to the point on the 
ROC curve closest to $(0, 1)$ in Euclidean distance, the $\alpha^*$ which 
maximizes the sum of estimated sensitivity and specificity, or the $\alpha^*$ 
which maximizes the estimated sensitivity for a fixed minimum specificity
estimate 
\citep[][]{lopez2014optimalcutpoints}. 
The choice of $\alpha^*$ will depend on the clinical application of interest. 
We classify an individual presenting with covariates $\bX = \bx$ 
as $\widehat{D}_{\alpha^*}(\bx)$. 
This is an equivalent formulation to the method proposed 
in Section 2.1 of \cite{veropoulos1999controlling}. 

\iffalse
\noindent
\begin{rmrk} \label{cost.rmrk}
The cost function
$C_A(\alpha)$ determines the relative weight placed on sensitivity
and specificity. We have that $C_1(\alpha) = \alpha$ and $C_{-1}(\alpha) = 1 - \alpha$. 
Thus, letting $\alpha$ go to 0 places
no weight on misclassification when $A = 1$ in truth, which results in
perfect specificity. Letting $\alpha$ go to 1 places no weight on
misclassification when $A = -1$ in truth, which results in perfect
sensitivity.
Increasing $\alpha$ places a harsher penalty on misclassification when $A = 1$ 
and thus will never decrease sensitivity. 
Similarly, decreasing $\alpha$ will never decrease specificity. 
In the extreme case where the training data are linearly separable, 
the estimated optimal classifier will not depend on $\alpha$. 
\end{rmrk}
\fi

\noindent
\begin{rmrk} \label{bayes.rmrk}
The optimal classifier over all functions mapping 
$\mathcal{X}$ into $\{-1, 1\}$, also known as the Bayes classifier \citep[][]{duda2012pattern}, 
can be expressed as 
\begin{equation} \label{bayes.rule}
D^*_\alpha(\bX) = \mathrm{sign} \big\{ \alpha \mathrm{Pr}(A = 1 | \bX) - (1 - \alpha) \mathrm{Pr}(A = -1 | \bX) \big\}.
\end{equation}
Thus, $D^*_\alpha(\bX)$ is equal to 1 when
%$$
%\frac{\mathrm{Pr}(A = 1 | \bX)}{\mathrm{Pr}(A = -1 | \bX)} > \frac{1-\alpha}{\alpha}
%$$
$\mathrm{Pr}(A = 1 | \bX) / \mathrm{Pr}(A = -1 | \bX) > (1 - \alpha) / \alpha$ 
or, using Bayes theorem, 
%$$
%\frac{p(\bX | A = 1)}{p(\bX | A = -1)} > \frac{(1-\alpha)(1-\rho)}{\alpha \rho} \equiv k_\alpha,
%$$
$p(\bX | A = 1) / p(\bX | A = -1) > (1-\alpha)(1-\rho) / \alpha \rho \equiv k_\alpha$, 
and $-1$ otherwise, where 
%$p(\bX | A = 1)$ and $p(\bX | A = -1)$ 
%are the conditional densities of $\bX$ when $A = 1$ and $A = -1$, respectively, and 
$\rho = \mathrm{Pr}(A = 1)$. Thus, the optimal classifier given in (\ref{bayes.rule}) has the same form as the 
Neyman--Pearson test of $H_0: A = -1$ against $H_1: A = 1$. 
%The type I error probability of this hypothesis test is equal to one minus specificity and power is equal to sensitivity. 
If we fix $k_\alpha$ (or equivalently, fix $\alpha$) to have fixed specificity $sp_0$,
%$$
%\mathrm{Pr}\left\{ \frac{p(X | A = 1)}{p(X | A = -1)} > k \Big| A = -1 \right\} = 1 - sp_0, 
%$$
then the Neyman--Pearson lemma ensures that $D^*_\alpha(\bX)$ maximizes sensitivity 
across all classifiers with specificity equal to $sp_0$. 
Therefore, the ROC curve for $D^*_\alpha(\bX)$, say $\mathrm{ROC}^*(u)$, has the property that 
$\mathrm{ROC}^*(u) \ge \mathrm{ROC}(u)$ for all $u \in (0, 1)$, where $\mathrm{ROC}(u)$ is the ROC 
curve corresponding to any other classifier. This is analogous to the result
given by \cite{mcintosh2002combining} 
\citep[see also page 71 of][]{pepe2003statistical}. 
\end{rmrk}

\noindent
\begin{rmrk} \label{true.opt.rmrk}
The optimal decision function in $\mathcal{F}$ is 
\begin{eqnarray*}
\widetilde{f}_\alpha & = & \operatorname*{arg\,min}_{f \in \mathcal{F}} 
\mathbb{E}\left( 1[ \mathrm{sign}\{f(\bX)\} \ne A] C_A(\alpha) \right) \\
 & = & \operatorname*{arg\,min}_{f \in \mathcal{F}} 
\left[ \rho \alpha \mathrm{Pr}\{f(\bX) < 0 | A = 1\} + (1-\rho)(1-\alpha)\mathrm{Pr}\{f(\bX) > 0 | A = -1\} \right] \\
 & = & \operatorname*{arg\,min}_{f \in \mathcal{F}} 
\left[ \rho \alpha \{1 - se(f)\} + (1-\rho)(1-\alpha)\{1 - sp(f)\} \right] \\
 & = & \operatorname*{arg\,max}_{f \in \mathcal{F}} 
\left\{ \rho \alpha se(f) + (1-\rho)(1-\alpha)sp(f) \right\},
\end{eqnarray*} 
where $se(f)$ and $sp(f)$ are the sensitivity and specificity of 
the decision rule $D = \mathrm{sign}(f)$. Thus, the true optimal decision function 
maximizes a weighted sum of sensitivity and specificity where the weights 
are determined by the population prevalence, $\rho$, and a user chosen weight, $\alpha$.  
\end{rmrk}

\subsection{Confidence Bands} \label{roc.conf.bands}

In this section, we present a method for 
constructing confidence bands for the ROC curve of $\widehat{f}_\alpha$, 
which provide an indication of how well the estimated classifier will perform in 
future samples. 
The method relies on consistency results given in Section~\ref{roc.theory} 
along with the following result, which characterizes the asymptotic distribution of 
the estimated sensitivity and specificity of $\widehat{f}_\alpha$. 
A proof is provided in Appendix A. 
\begin{thm} \label{se.limit}
Let $se\left(\widehat{f}_\alpha\right)$ be the true sensitivity, 
$\widehat{se}\left(\widehat{f}_\alpha\right)$ be the estimated 
sensitivity, $sp\left(\widehat{f}_\alpha\right)$ be the 
true specificity, and $\widehat{sp}\left(\widehat{f}_\alpha\right)$ 
be the estimated specificity of $\widehat{f}_\alpha$, 
where $\widehat{f}_\alpha$ is defined in~(\ref{wsvm}), and assume that 
$\mathcal{F}$ is a space of linear or polynomial functions. 
Define $f_{\alpha, \phi}^* = \operatorname*{arg \, min}_f \mathbb{E} 
\left[ \phi\left\{A f(\bX)\right\} C_A(\alpha)\right]$, where the minimization 
is taken over all measurable functions mapping $\mathcal{X}$ into $\mathbb{R}$, 
and assume that $f_{\alpha, \phi}^* \in \mathcal{F}$. 
Let $\widetilde{f}_\alpha$ be defined as in Remark~\ref{true.opt.rmrk} and let $\tau = 1 - \rho$.
Then, 
\begin{equation*}
\sqrt{n} \left\{ \begin{array}{c} \widehat{se}\left(\widehat{f}_\alpha\right) 
 - se\left(\widehat{f}_\alpha\right) \\ \widehat{sp}\left(\widehat{f}_\alpha\right) 
 - sp\left(\widehat{f}_\alpha\right) \end{array} \right\} 
\rightsquigarrow \left\{ \begin{array}{c} \mathbb{G}_1(\alpha) \\ \mathbb{G}_2(\alpha) \end{array} \right\}
\end{equation*}
as $n \rightarrow \infty$, 
where $\mathbb{G}_1(\alpha)$ and $\mathbb{G}_2(\alpha)$ are mean zero Gaussian processes with covariances 
\begin{multline*}
\sigma_1(\alpha_1, \alpha_2) = \mathbb{E} \left( \rho^{-2} 1(A = 1) 
\left[1\left\{\widetilde{f}_{\alpha_1}(\bX) > 0\right\} - se\left(\widetilde{f}_{\alpha_1}\right)\right]
\left[1\left\{\widetilde{f}_{\alpha_2}(\bX) > 0\right\} - se\left(\widetilde{f}_{\alpha_2}\right)\right]\right) \\
 - \mathbb{E}\left( \rho^{-1} 1(A = 1) 
\left[1\left\{\widetilde{f}_{\alpha_1}(\bX) > 0\right\} - se\left(\widetilde{f}_{\alpha_1}\right)\right]\right) 
\mathbb{E}\left( \rho^{-1} 1(A = 1) 
\left[1\left\{\widetilde{f}_{\alpha_2}(\bX) > 0\right\} - se\left(\widetilde{f}_{\alpha_2}\right)\right]\right), 
\end{multline*}
and 
\begin{multline*}
\sigma_2(\alpha_1, \alpha_2) = \mathbb{E} \left( \tau^{-2} 1(A = -1) 
\left[1\left\{\widetilde{f}_{\alpha_1}(\bX) < 0\right\} - sp\left(\widetilde{f}_{\alpha_1}\right)\right]
\left[1\left\{\widetilde{f}_{\alpha_2}(\bX) < 0\right\} - sp\left(\widetilde{f}_{\alpha_2}\right)\right]\right) \\
 - \mathbb{E}\left( \tau^{-1} 1(A = -1) 
\left[1\left\{\widetilde{f}_{\alpha_1}(\bX) < 0\right\} - sp\left(\widetilde{f}_{\alpha_1}\right)\right]\right) 
 \mathbb{E}\left( \tau^{-1} 1(A = -1) 
\left[1\left\{\widetilde{f}_{\alpha_2}(\bX) < 0\right\} - sp\left(\widetilde{f}_{\alpha_2}\right)\right]\right), 
\end{multline*}
respectively, with cross-covariance 
\begin{multline*}
\sigma_{12}(\alpha_1, \alpha_2) = - \mathbb{E}\left( \rho^{-1} 1(A = 1) 
\left[1\left\{\widetilde{f}_{\alpha_1}(\bX) > 0\right\} - se\left(\widetilde{f}_{\alpha_1}\right)\right]\right) \\
 \times \mathbb{E}\left( \tau^{-1} 1(A = -1) 
\left[1\left\{\widetilde{f}_{\alpha_2}(\bX) < 0\right\} - sp\left(\widetilde{f}_{\alpha_2}\right)\right]\right). 
\end{multline*}
\end{thm}

Let $fpf\left(\widehat{f}_\alpha\right) = 1 - sp\left(\widehat{f}_\alpha\right)$ 
be the false positive fraction for the decision function $\widehat{f}_\alpha$. Define $fpf^{-1}(\cdot)$ such that 
$fpf^{-1}\left\{fpf\left(\widehat{f}_\alpha\right)\right\} = \alpha$, i.e., $fpf^{-1}(u)$ is the weight $\alpha$ 
such that $1 - sp\left(\widehat{f}_\alpha\right) = u$. Let $0 < \delta < 1/2$ be fixed. 
A quantile bootstrap algorithm for constructing an asymptotically correct 
$(1 - \gamma)100\%$ confidence band for the ROC curve, 
$se\left\{fpf^{-1}(u)\right\}$, $\delta < u < 1$, is as follows: 
\begin{enumerate}
\item Set a large number of bootstrap replications, B, a grid $\delta = z_1 < \ldots < z_K = 1$ 
and a grid $0 = \alpha_1 < \ldots < \alpha_M = 1$.
\item For $m = 1, \ldots, M$, compute the estimated ROC curve, 
$\widehat{R}(\alpha_m) = \left\{1 - \widehat{sp}\left(\widehat{f}_{\alpha_m}\right), 
\widehat{se}\left(\widehat{f}_{\alpha_m}\right)\right\}$.
\item For $k = 1, \ldots, K$, compute $\widehat{y}(z_k)$ by linearly interpolating $\widehat{R}(\alpha_m)$.
\item For $b = 1, \ldots, B$:
\begin{enumerate}
\item Generate a weight vector $W_{b, n, i} = \xi_{b, i} / \bar{\xi}_b$, 
where $\xi_{b, 1}, \ldots, \xi_{b, n}$ are independent standard 
exponential random variables and $\bar{\xi}_b = n^{-1} \sum_{i = 1}^n \xi_{b, i}$.
\item For $m = 1, \ldots, M$, set 
\begin{equation*}
\widetilde{se}_b\left(\widehat{f}_\alpha\right) = \mathbb{E}_n \left(W_{b, n} 1\left(A = 1\right) 
1\left[ \mathrm{sign}\left\{\widehat{f}_\alpha (\bX)\right\} = 1\right] \right)
 \big/ \mathbb{E}_n \left\{W_{b, n} 1(A = 1) \right\},
\end{equation*}
%and 
\begin{equation*}
\widetilde{sp}_b\left(\widehat{f}_\alpha\right) = \mathbb{E}_n \left(W_{b, n} 1\left(A = -1\right) 
1\left[ \mathrm{sign}\left\{\widehat{f}_\alpha (\bX)\right\} = -1\right] \right)
 \big/ \mathbb{E}_n \left\{W_{b, n} 1(A = -1) \right\},
\end{equation*}
and $\widetilde{R}_b(\alpha_m) = \left\{1 - \widetilde{sp}_b\left(\widehat{f}_{\alpha_m}\right), 
\widetilde{se}_b\left(\widehat{f}_{\alpha_m}\right)\right\}$.
\item For $k = 1, \ldots, K$, compute $\widetilde{y}_b(z_k)$ by linearly interpolating $\widetilde{R}_b(\alpha_m)$.
\end{enumerate}
\item Let $\widetilde{y}^p(z_k)$ be the $p$-th quantile of $\left\{\widetilde{y}_b(z_k) : b = 1, \ldots, B\right\}$
and let $p^*$ be the largest $p \in [0, 1]$ such that $\widetilde{y}^{p^*/2}(z_k) \le \widetilde{y}_b(z_k) 
\le \widetilde{y}^{1 - p^*/2}(z_k)$ for all $k = 1, \ldots, K$ for at least $(1 - \gamma)B$ bootstrap samples.
\item Set $y_\ell(z_k) = \widetilde{y}^{p^*/2}(z_k)$ and $y_u(z_k) = \widetilde{y}^{1 - p^*/2}(z_k)$.
%\item Set $\overline{y}(x_k) = \mathrm{median}\left\{\widetilde{y}_b(x_k) : b = 1, \ldots, B\right\}$
%\item Set $y_\ell(x_k) = \widehat{y}(x_k) - \widetilde{y}^{1 - p/2}(x_k) + \overline{y}(x_k)$ 
%and $y_u(x_k) = \widehat{y}(x_k) - \widetilde{y}^{p/2}(x_k) + \overline{y}(x_k)$
\end{enumerate}
One can also use alternate choices for the weights, for example, 
a multinomial weight vector $W_{b, n} = \left(W_{b, n, 1}, \ldots, W_{b, n, n}\right)^\intercal$
with probabilities $(1/n, \ldots, 1/n)$ and $n$ trials. 
Let $\underset{W}{\overset{P}{\rightsquigarrow}}$ denote convergence in probability over $W$, as defined in 
Section 2.2.3 and Chapter 10 of \cite{kosorok2008introduction}. The following result states the 
consistency of the bootstrap. 
\begin{cor} \label{boot.cor}
Let $\widetilde{se}_W\left(\widehat{f}_\alpha\right) = \mathbb{E}_n \left(W 1\left(A = 1\right) 
1\left[ \mathrm{sign}\left\{\widehat{f}_\alpha (\bX)\right\} = 1\right] \right) 
 \big/ \mathbb{E}_n \left\{W 1(A = 1) \right\}$ 
and define $\widetilde{sp}_W\left(\widehat{f}_\alpha\right)$ 
similarly. Let $\widetilde{R}_W(\alpha) = \left\{ 1 - \widetilde{sp}_W\left(\widehat{f}_\alpha\right), 
\widetilde{se}_W\left(\widehat{f}_\alpha\right)\right\}$ and let $\widehat{R}\left(\alpha\right)$ be 
as defined above. Then, for any $0 < \delta < 1/2$, $\widetilde{R}_W(\alpha) 
\underset{W}{\overset{P}{\rightsquigarrow}} \widehat{R}(\alpha)$ 
in $\ell^\infty\left([\delta, 1]\right)$. 
\end{cor}
\begin{proof}
By Lemmas~12.7 and 12.8 of \cite{kosorok2008introduction}, taking the inverse of a bounded, monotone 
function is Hadamard differentiable under mild regularity conditions. 
The result now follows by Theorem~\ref{se.limit} above and Theorems~2.6 and 2.9 of \cite{kosorok2008introduction}.
\end{proof}

Thus, $\left\{y_\ell(z_k), y_u(z_k) \right\}$ will cover $\widehat{y}(z_k)$ across $k = 1, \ldots, K$ 
with probability $1 - \gamma$ for large enough $n$ and $B$.
In addition to the linear and polynomial SVM, 
this procedure will work for any classifier such that the estimated decision 
function is in a VC class, such as a logistic regression classifier. 

\section{Uniform Consistency} \label{roc.theory}

%In this section, we present several theoretical results pertaining to the weighted SVM. 
For any $\alpha \in (0, 1)$, the estimated classifier is the sign of 
$\widehat{f}_\alpha$, the minimizer of the 
empirical hinge loss in a class $\mathcal{F}$ as defined in (\ref{wsvm}). 
For any function, $f$, define 
$\mathcal{R}_\alpha(f) = \mathbb{E}\left( 1[\mathrm{sign}\{f(\bX)\} \ne A] C_A(\alpha) \right)$
to be the risk of $f$, and define the $\phi$ risk of $f$ to be 
$\mathcal{R}_{\alpha, \phi}(f) = \mathbb{E}\left[\phi\{Af(\bX)\} C_A(\alpha)\right]$. 
Let $\mathcal{R}^*_\alpha = \mathrm{inf}_f \mathcal{R}_\alpha(f)$ 
and $\mathcal{R}^*_{\alpha, \phi} = \mathrm{inf}_f \mathcal{R}_{\alpha, \phi}(f)$. 
Furthermore, define $\widetilde{f}_\alpha = \mathrm{arg \, min}_{f \in \mathcal{F}} \mathcal{R}_\alpha(f)$ 
and $f^*_\alpha = \mathrm{arg \, min}_f \mathcal{R}_\alpha(f)$, i.e., 
$\widetilde{f}_\alpha$ minimizes the risk over $\mathcal{F}$ and 
$f^*_\alpha$ minimizes the risk over all measurable functions mapping $\mathcal{X}$ into $\mathbb{R}$. 
Define $f_{\alpha, \phi}^*$ as in Theorem~\ref{se.limit}. 
%Finally, define $f_{\alpha, \phi}^* = \mathrm{arg \, min}_f \mathcal{R}_{\alpha, \phi}(f)$. 
%In the following, we assume that $\mathcal{F}$ is chosen so that 
%both $\widehat{f}_\alpha$ exists and is unique. 
%None of our results require that the true minimizer, $f^*_\alpha$, 
%be contained in $\mathcal{F}$; thus, we are able to consistently estimate 
%the best classifier in the class even when the model is not correctly specified.
Throughout, we assume that $f^*_{\alpha, \phi} \in \mathcal{F}$, i.e., that the 
function that minimizes the $\phi$ risk is contained within the chosen class. 
If this is not the case, the consistency results given here will not hold; however, 
the estimated decision function will still yield a reasonable approximation to 
$\widetilde{f}_{\alpha}$ due to the identity $\mathcal{R}_\alpha(f) \le \mathcal{R}_{\alpha, \phi}(f)$. 
When $\alpha = 0$, the optimal classifier assigns $-1$ uniformly and 
when $\alpha = 1$, the optimal classifier assigns 1 uniformly. 
Focusing on $\alpha \in (0, 1)$ will enable us to avoid these trivial extremes. 
Nonetheless, many of our results hold for all $\alpha \in [0, 1]$. 
We will make this distinction explicit as needed. 
Throughout, we assume that all requisite expectations exist.  

%\subsection{Excess Risk and Consistency}

The following result gives a bound on the excess risk in terms of the excess $\phi$ risk. 
The proof is similar to that of Theorem~3.2 of 
\cite{zhao2012estimating} and uses Theorem~1 and Example~4 of \cite{bartlett2006convexity}. 
We omit the proof here. This result will be used later to show 
uniform consistency of the risk of the estimated decision function. 

\noindent
\begin{lem}\label{excessrisk}
For any measurable $f:\mathcal{X} \rightarrow \mathbb{R}$ and any distribution $P$ of $(\bX, A)$, 
$\mathcal{R}_\alpha (f) - \mathcal{R}_\alpha^* \le \mathcal{R}_{\alpha, \phi}(f) - \mathcal{R}_{\alpha,\phi}^*$. 
\end{lem}

\noindent
This result implies that the difference between the $\phi$ 
risk of the estimated decision function and the 
optimal $\phi$ risk is no smaller than the difference 
between the risk of the estimated decision function and the 
optimal risk. Therefore, we can consider the 
$\phi$ risk when proving convergence results. 

%\subsection{Consistency}

Next, we establish a number of consistency results for the risk of the 
estimated decision function. We begin with Fisher consistency. 
This result implies that estimation using either the hinge
loss or the zero-one loss will yield the true optimal classifier given an infinite
sample, providing justification for using the proposed surrogate loss function. 
The proof follows from an extension to the proof of Proposition~3.1 of 
\cite{zhao2012estimating} and is in Appendix A. 

\noindent
\begin{thm}\label{fisher}
For any $\alpha \in [0, 1]$, if $f^*_{\alpha, \phi}$ minimizes $\mathcal{R}_{\alpha, \phi}$, then 
$D^*_\alpha(\bx) = \mathrm{sign}\big\{ f^*_{\alpha, \phi}(\bx) \big\}$ for almost all $\bx \in \mathcal{X}$. 
\end{thm}

The following result establishes consistency of the risk of the estimated 
decision function when estimation takes place within a RKHS. 
We then extend this consistency by showing that it is uniform in $\alpha$. 
The proof of the following result 
closely follows the proof of Theorem~3.3 of \cite{zhao2012estimating} and is in Appendix A. 

\noindent
\begin{thm} \label{cons}
Let $\alpha \in [0, 1]$ be fixed and let $\lambda_n$ be a sequence of positive, real numbers 
such that $\lambda_n \rightarrow 0$ and $n \lambda_n \rightarrow \infty$. 
Let $\mathcal{H}_k$ be a RKHS with kernel function $k$ and let $\bar{\mathcal{H}}_k$ 
denote the closure of $\mathcal{H}_k$.  
Then, for any distribution $P$ of $(\bX, A)$, we have that 
$\left|\mathcal{R}_\alpha\left(\widehat{f}_\alpha\right) 
- \inf_{f \in \bar{\mathcal{H}}_k} \mathcal{R}_{\alpha}(f)\right|
\xrightarrow[]{P} 0$ 
as $n \rightarrow \infty$. 
\end{thm}

We next strengthen the consistency stated above by showing that 
the convergence is uniform in $\alpha$ when estimation uses a linear, 
quadratic, polynomial, or Gaussian kernel 
\citep[see][for a discussion of kernel functions used with the SVM]{steinwart2008support}. 
%This result implies that deviating from a chosen 
%$\alpha^*$ will not result in gains in small sample performance.
The following lemma indicates that the estimated decision function 
lies in a Glivenko--Cantelli class \citep[][]{kosorok2008introduction} 
indexed by $\alpha$, which will help us to extend the consistency stated above 
to uniform consistency in $\alpha$. The proof is in Appendix A.

\noindent
\begin{lem} \label{gc} 
Let $\widehat{f}_\alpha$ be estimated using a linear, quadratic, polynomial, 
or Gaussian kernel function. Then, $\left\{\widehat{f}_\alpha : \alpha \in [0, 1]\right\}$ 
is contained in a Glivenko--Cantelli (GC) class. 
\end{lem}

\noindent
Given that $\widehat{f}_\alpha$ and $-\widehat{f}_\alpha$ are 
contained in a GC class, we have by Corollary 9.27 (iii) 
of \cite{kosorok2008introduction}, that $\phi\left(\widehat{f}_\alpha\right)$ 
and $\phi\left(-\widehat{f}_\alpha\right)$ are contained in a GC class because $\phi$ is continuous. 
By Corollary 9.27 (ii) of \cite{kosorok2008introduction}, 
$1(A = 1)\phi\left(\widehat{f}_\alpha\right)$ and $1(A = -1)\phi\left(-\widehat{f}_\alpha\right)$ 
are contained in a GC class and thus, $L_{\alpha, \phi}\left(\widehat{f}_{\alpha}\right)$ is 
contained in a GC class by Corollary 9.27 (i) of \cite{kosorok2008introduction}, 
where $L_{\alpha, \phi}(f) = \phi(Af)C_A(\alpha)$. It follows that 
$\sup_{\alpha \in [0, 1]} \left| \widehat{\mathcal{R}}_{\alpha, \phi}\left(\widehat{f}_\alpha\right) 
- \mathcal{R}_{\alpha, \phi}\left(\widehat{f}_\alpha\right) \right| \xrightarrow[]{P} 0$, 
where $\widehat{\mathcal{R}}_{\alpha, \phi}(f) = \mathbb{E}_n \phi\{A f(\bX)\} C_A(\alpha)$. 
This convergence will be used in the proof of Theorem~\ref{unifcons}, which is given in Appendix A. 

\noindent
\begin{thm} \label{unifcons}
Assume that $\widehat{f}_\alpha$ is estimated using a linear, 
quadratic, polynomial, or Gaussian kernel. 
For any sequence $\lambda_n$ of positive, real numbers satisfying 
$\lambda_n \rightarrow 0$ and $n \lambda_n \rightarrow \infty$ 
and any distribution $P$ of $(\bX, A)$, 
\begin{equation} \label{unif.cons.eqn}
\operatorname*{sup}_{\alpha \in [0, 1]} \left| \mathcal{R}_\alpha \left(\widehat{f}_\alpha\right) 
- \inf_{f \in \bar{\mathcal{H}}_k} \mathcal{R}_\alpha(f) \right| \xrightarrow[]{P} 0
\end{equation}
as $n \rightarrow \infty$, where $\mathcal{H}_k$ is the RKHS associated with $\widehat{f}_\alpha$. 
\end{thm}

\noindent
Note that we do not allow the sequence $\lambda_n$ to depend on $\alpha$, which is reflected in the 
implementation in Section~\ref{roc.simul} below. 

%\subsection{Continuity} 

Here, we prove a number of continuity and convergence results 
regarding the ROC curve and risk function for $\widetilde{f}_\alpha$ and $\widehat{f}_\alpha$. 
We begin with the following result which indicates that the 
ROC curve of the Bayes classifier, $D^*_\alpha$, is continuous. 
We require $\mathrm{Pr}(A = 1 | \bX)$ to be a continuous random variable; 
however, we do not require that the map $\bx \mapsto \mathrm{Pr}(A = 1 | \bX = \bx)$ be continuous. 
The proof is included in Appendix A. 

\noindent
\begin{lem}\label{continuity} 
Let $se^*(\alpha) = \mathrm{Pr}\{D^*_\alpha(\bX) = 1 | A = 1\}$ 
and $sp^*(\alpha) = \mathrm{Pr}\{D^*_\alpha(\bX) = -1 | A = -1\}$ 
be the sensitivity and specificity of $D^*_\alpha$. Then, $se^*(\alpha)$ and $sp^*(\alpha)$ 
are continuous in $\alpha$ whenever $\mathrm{Pr}(A = 1 | \bX)$ is a continuous 
random variable with support $(0, 1)$. 
\end{lem}

\noindent
Thus, $\mathrm{ROC}^*(u)$ is monotone nondecreasing and continuous except possibly at 0. 
It follows from Lemma~\ref{continuity} and Remark~\ref{true.opt.rmrk} 
that $\mathcal{R}_\alpha(f^*_\alpha)$ is continuous in $\alpha$. 
This is used in the proof of the following result, which is deferred to Appendix A. 

\noindent
\begin{thm}\label{risk.continuity}
Under the assumptions of Lemma~\ref{continuity}, 
%the risk function of the $\phi$ risk minimizer over $\mathcal{F}$, 
$\mathcal{R}_\alpha\left(\widetilde{f}_\alpha\right)$, is continuous in $\alpha$. 
\end{thm}

\iffalse
\noindent
This result indicates that the risk of the the true best classifier in $\mathcal{F}$ 
%, $\mathcal{R}_\alpha\left(\widetilde{f}_\alpha\right)$, 
is continuous in $\alpha$. 
\fi

Finally, we state two corollaries pertaining to the sensitivity and specificity 
of the estimated decision rule. These results show that 
the ROC curve of the estimated decision function converges uniformly to the ROC curve 
of the optimal decision function in $\mathcal{F}$. 
The proof of Corollary~\ref{cor2} relies on a novel empirical process result 
which is included in Appendix A. 
\begin{cor} \label{cor1}
Let $\widehat{f}_\alpha$ be estimated using a linear, quadratic, polynomial, 
or Gaussian kernel function.
Let $se\left(\widehat{f}_\alpha\right) = \mathrm{Pr}\left\{\widehat{f}_\alpha(\bX) > 0 | A = 1 \right\}$ 
be the sensitivity 
and $sp\left(\widehat{f}_\alpha\right) = \mathrm{Pr}\left\{\widehat{f}_\alpha(\bX) < 0 | A = -1 \right\}$ 
be the specificity of the decision rule 
$\widehat{d}_\alpha = \mathrm{sign}\left(\widehat{f}_\alpha\right)$. 
Then, there exist $\widetilde{se}_\alpha$ and $\widetilde{sp}_\alpha$ such that 
$\alpha \rho \widetilde{se}_\alpha + (1 - \alpha)(1 - \rho) \widetilde{sp}_\alpha 
 = \alpha \rho se\left(\widetilde{f}_\alpha\right) + (1 - \alpha)(1 - \rho) sp\left(\widetilde{f}_\alpha \right)$ 
and 
$\mathrm{sup}_{\alpha \in [0, 1]} \left| se\left(\widehat{f}_\alpha\right)
- \widetilde{se}_\alpha \right| \xrightarrow[]{P} 0$
and 
$\mathrm{sup}_{\alpha \in [0, 1]} \left| sp\left(\widehat{f}_\alpha\right)
- \widetilde{sp}_\alpha \right| \xrightarrow[]{P} 0$ as $n \rightarrow \infty$. 
\end{cor}

Note that Corollary~\ref{cor1} does not require $\widetilde{f}_\alpha$ 
to be unique. We can only say that the sensitivity and specificity of 
$\widehat{f}_\alpha$ converge to the sensitivity and specificity of 
a function in the same equivalence class as $\widetilde{f}_\alpha$, 
i.e., a function with optimal risk. 

\noindent
\begin{cor} \label{cor2}
Let $\widehat{se}\left(\widehat{f}_\alpha\right)$
and $\widehat{sp}\left(\widehat{f}_\alpha\right)$ 
be defined as in Section~\ref{roc.setup}. 
Let $\widetilde{se}_\alpha$ and $\widetilde{sp}_\alpha$ 
be as defined in Corollary~\ref{cor1}. Then, 
$\mathrm{sup}_{\alpha \in [0, 1]} \left| \widehat{se}\left(\widehat{f}_\alpha\right)
- \widetilde{se}_\alpha \right| \xrightarrow[]{P} 0$, 
and $\mathrm{sup}_{\alpha \in [0, 1]} \left| \widehat{sp}\left(\widehat{f}_\alpha\right)
- \widetilde{sp}_\alpha \right| \xrightarrow[]{P} 0$ as $n \rightarrow \infty$. 
\end{cor}

\iffalse
\noindent
Thus, the ROC curve of $\widehat{f}_\alpha$ converges uniformly 
to the ROC curve of $\widetilde{f}_\alpha$. 
\fi

\section{Simulation Experiments} \label{roc.simul}

To investigate the performance of classification using a weighted SVM 
and the resulting ROC curves and confidence bands, we use the following generative model. 
Let $\bX$ be generated according to $\bX \sim N_p(\mu \bZ, \sigma^2 I)$, where 
$\bZ$ is equal to a vector of ones with probability $q$ and a vector 
of negative ones with probability $1 - q$ and $I$ is a $p \times p$ identity matrix.  
Thus, $\bX$ is a mixture of multivariate normal distributions with mixing probability $q$. 
Let $\pi(\bX) = \mathrm{expit}\left(\bX^\intercal \beta\right)$ for a $p \times 1$ vector $\beta$, 
where $\mathrm{expit}(u) = \mathrm{exp}(u) / \left\{1 + \mathrm{exp}(u)\right\}$. 
Given $\bX$, we let $A$ be equal to 1 with probability $\pi(\bX)$ and $-1$ with probability 
$1 - \pi(\bX)$. Because $\pi(\bX)$ depends on $\bX$ only through a linear function of $\bX$, 
we refer to this model below as the linear generative model. 
We also consider a generalization of the above model where 
$\pi(\bX) = \mathrm{expit}\left(\bX^\intercal \beta + X_1^2 + X_2^2 + 4 X_1 X_2\right)$, 
which we refer to below as the nonlinear generative model. 

We implement the weighted SVM in MATLAB software using 
the LIBSVM library of \cite{chang2011libsvm}. 
Each simulated data set is divided into training and testing sets with 
70\% of the data used for training the SVM and 30\% used to estimate sensitivity and 
specificity. We use both linear and Gaussian kernels. 
The Gaussian kernel function is 
$k(\bx, \mathbf{y}) = \mathrm{exp}(-\gamma \|\bx-\mathbf{y}\|^2 )$
\citep[][]{steinwart2008support}. 
The bandwidth parameter, $\gamma$, and the penalty parameter, $\lambda_n$, are 
estimated using cross-validation within the training data for $\alpha = 0.5$ and the resulting 
tuning parameters are used to fit the weighted SVM for all $\alpha$
on a grid over $(0, 1)$. 
Comparison methods are implemented in R software \citep[][]{rcoreteam}. 

We compare the performance of the weighted SVM 
to standard methods in diagnostic medicine, including 
logistic regression \citep[][]{mcintosh2002combining} 
and semiparametric ROC curves \citep[][]{cai2004semi}. 
Logistic regression and the SVM combine multiple 
biomarkers while the semiparametric ROC curve is calculated 
for a single biomarker (the first component of $\bX$). 
These four methods are applied to simulated data from the 
linear and nonlinear generative models with 
$n = 250, 500$, $p = 2, 5, 10$, $q = 0.05, 0.25$, 
$\sigma = 0.75$, and $\mu = 0.25$. When $p = 2, 5$, we use $\beta = (2, 1)^\intercal$ 
and $\beta = (2, 1, \ldots, 1)^\intercal$, respectively. 
When $p = 10$, we use $\beta = (2, 1, 1, 1, 1, 0, \ldots, 0)^\intercal$, 
i.e., noise variables are introduced for the case where $p = 10$. 
We report the mean area under the ROC curve (AUC) 
and the Monte Carlo standard deviation of AUC 
as well as optimal sensitivity and specificity 
across 100 replications. 
Optimal sensitivity and specificity are calculated as the point 
on the ROC curve closest to $(0, 1)$ in Euclidean distance 
\citep[see][for a discussion of different methods 
for selecting the optimal point on the ROC curve]{lopez2014optimalcutpoints}. 
 
Table~\ref{sim1.auc.nonlin} below contains estimated AUCs 
averaged across replications and Monte Carlo standard 
deviations of AUCs for the four methods 
when the true generative model is nonlinear. 
%\input{roc.sim1.auc.nonlin.tables.txt}
% latex table generated in R 3.2.2 by xtable 1.7-4 package
% Tue Apr 18 17:50:09 2017
\begin{table}[h!]
\caption{Average AUC when true model is nonlinear.} 
\label{sim1.auc.nonlin}
\centering
\begin{tabular}{ccc|cccc}
   \hline
$n$ & $p$ & $q$ & Linear SVM & Gaussian SVM & Logistic & Semiparametric \\ 
   \hline
250 & 2 & 0.05 & 0.61 (0.07) & 0.78 (0.06) & 0.58 (0.08) & 0.58 (0.04) \\ 
   &  & 0.25 & 0.64 (0.07) & 0.81 (0.05) & 0.62 (0.06) & 0.62 (0.03) \\ 
   & 5 & 0.05 & 0.71 (0.06) & 0.75 (0.06) & 0.71 (0.07) & 0.56 (0.03) \\ 
   &  & 0.25 & 0.74 (0.05) & 0.77 (0.06) & 0.74 (0.06) & 0.62 (0.03) \\ 
   & 10 & 0.05 & 0.70 (0.06) & 0.56 (0.05) & 0.70 (0.06) & 0.57 (0.04) \\ 
   &  & 0.25 & 0.74 (0.06) & 0.56 (0.05) & 0.74 (0.06) & 0.62 (0.04) \\ 
   \hline
500 & 2 & 0.05 & 0.61 (0.05) & 0.81 (0.04) & 0.59 (0.05) & 0.58 (0.02) \\ 
   &  & 0.25 & 0.65 (0.04) & 0.81 (0.04) & 0.61 (0.05) & 0.61 (0.02) \\ 
   & 5 & 0.05 & 0.72 (0.04) & 0.78 (0.04) & 0.72 (0.05) & 0.57 (0.02) \\ 
   &  & 0.25 & 0.77 (0.04) & 0.80 (0.04) & 0.75 (0.03) & 0.62 (0.02) \\ 
   & 10 & 0.05 & 0.71 (0.04) & 0.60 (0.05) & 0.71 (0.04) & 0.56 (0.03) \\ 
   &  & 0.25 & 0.75 (0.04) & 0.60 (0.04) & 0.74 (0.04) & 0.62 (0.02) \\ 
   \hline
\end{tabular}
\end{table}
The Gaussian SVM outperforms the other methods except in the case 
where there are noise variables. The linear SVM slightly 
outperforms logistic regression in most cases.
%\ref{sim1.sesp.nonlin}
Table~\ref{sim1.sesp.nonlin} in Appendix B contains optimal 
sensitivities and specificities for the four methods
when the true generative model is nonlinear, averaged across replications.
%\ref{sim1.sesp.nonlin.uw}
Table~\ref{sim1.sesp.nonlin.uw} contains estimated sensitivities and specificities 
of an unweighted SVM when the true model is nonlinear.
The unweighted SVM often fails to achieve a balance between sensitivity and 
specificity. In particular, the linear SVM often achieves low specificity. 
The imbalance between sensitivity and specificity is often worse 
when $q$ is small, indicating that proper balance is difficult 
to achieve when there is an imbalance between true class labels 
in the data. These results highlight the importance 
of estimating the full ROC curve and selecting the weight 
to achieve the desired balance between sensitivity and specificity; 
unweighted classification may not achieve satisfactory performance in many settings. 
%\ref{sim1.auc.lin} and \ref{sim1.sesp.lin}
Tables~\ref{sim1.auc.lin}, \ref{sim1.sesp.lin}, and \ref{sim1.sesp.lin.uw} in 
Appendix B contain results when the true generative model is linear. 
%\ref{sim1.sesp.lin.uw}

Next, we examine the performance of the proposed bootstrap confidence band method 
for the linear SVM. 
%Using the same simulated data that was generated for Tables~\ref{sim1.auc.nonlin} and \ref{sim1.auc.lin}, 
Independent testing sets of size 100,000 were used to calculate 
$se\left(\widehat{f}_\alpha\right)$ and $sp\left(\widehat{f}_\alpha\right)$, giving us 
an approximation to the true ROC curve for each $\widehat{f}_\alpha$. 
%Then, we 
%construct the ROC curve confidence band using the method introduced in Section~\ref{roc.conf.bands}. 
The method introduced in Section~\ref{roc.conf.bands} was used to 
construct 90\% confidence bands using 1000 bootstrap samples.  
We report the proportion of 100 Monte Carlo replications for which the 
true ROC curve is fully contained within the confidence band across [0.01, 0.99] along with 
the average area between the upper and lower confidence bands. 
Table~\ref{tab.cov.prob} contains these results. 
%\input{coverage.prob.quant.txt} 
% latex table generated in R 3.3.1 by xtable 1.8-2 package
% Mon Nov 13 17:45:51 2017
\begin{table}[h!]
\caption{Estimated coverage probabilities and area between confidence band curves.} 
\label{tab.cov.prob}
\centering
\begin{tabular}{ccc|cccc}
  \hline
   & & & \multicolumn{2}{c}{Coverage probability} & \multicolumn{2}{c}{Area between curves} \\$n$ & $p$ & $q$ & Linear model & Nonlinear model & Linear model & Nonlinear model \\ 
   \hline
250 & 2 & 0.05 & 0.92 & 0.89 & 0.31 & 0.36 \\ 
   &  & 0.25 & 0.93 & 0.89 & 0.29 & 0.37 \\ 
   & 5 & 0.05 & 0.91 & 0.93 & 0.27 & 0.38 \\ 
   &  & 0.25 & 0.83 & 0.97 & 0.25 & 0.37 \\ 
   & 10 & 0.05 & 0.88 & 0.93 & 0.29 & 0.39 \\ 
   &  & 0.25 & 0.64 & 0.90 & 0.26 & 0.37 \\ 
   \hline
500 & 2 & 0.05 & 0.97 & 0.94 & 0.24 & 0.27 \\ 
   &  & 0.25 & 0.97 & 0.91 & 0.22 & 0.27 \\ 
   & 5 & 0.05 & 0.95 & 0.97 & 0.20 & 0.29 \\ 
   &  & 0.25 & 0.94 & 0.94 & 0.18 & 0.26 \\ 
   & 10 & 0.05 & 0.94 & 0.93 & 0.22 & 0.29 \\ 
   &  & 0.25 & 0.87 & 0.97 & 0.19 & 0.27 \\ 
   \hline
\end{tabular}
\end{table}
We observe that, across $n$, $p$, and $q$, the proposed quantile bootstrap method
provides approximately 90\% coverage with the area between curves decreasing for larger sample sizes. 

Figure~\ref{roc.cb} below contains bootstrap confidence bands for one simulated replication for 
the linear and nonlinear generative model when $n = 500$, $p = 2$, and $q = 0.25$. 
%Figure~\ref{lin.roc.cb} contains confidence bands for the ROC curve 
%when the true model is linear and Figure~\ref{nonlin.roc.cb} contains confidence bands for the 
%ROC curve when the true model is nonlinear. 
The true ROC curve, calculated from a large testing set of size 100,000, is also plotted. 
\begin{figure}[h!]
\centering
\subfigure{\includegraphics[width=.45\linewidth]{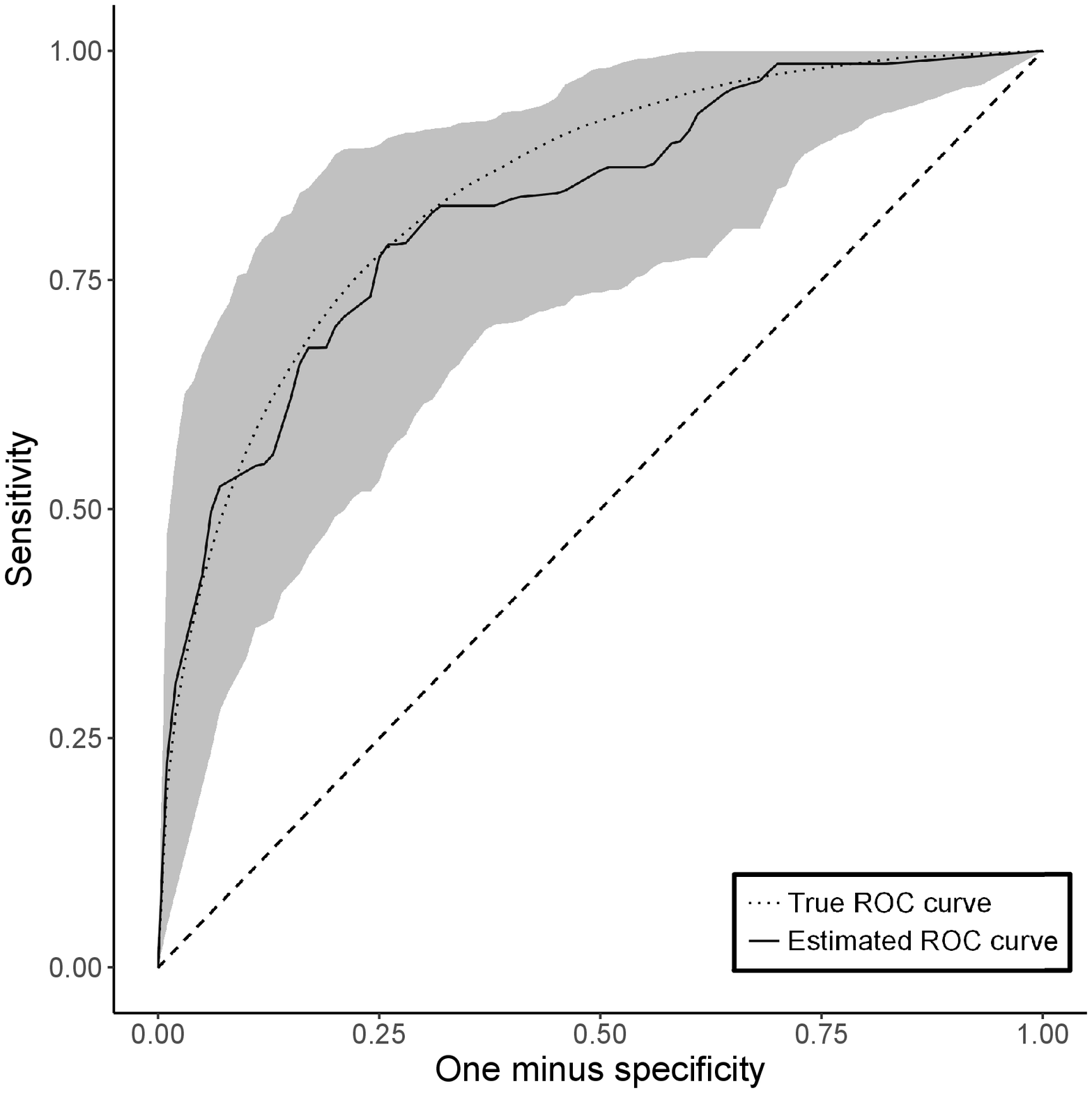}}%
\hfill
\subfigure{\includegraphics[width=.45\linewidth]{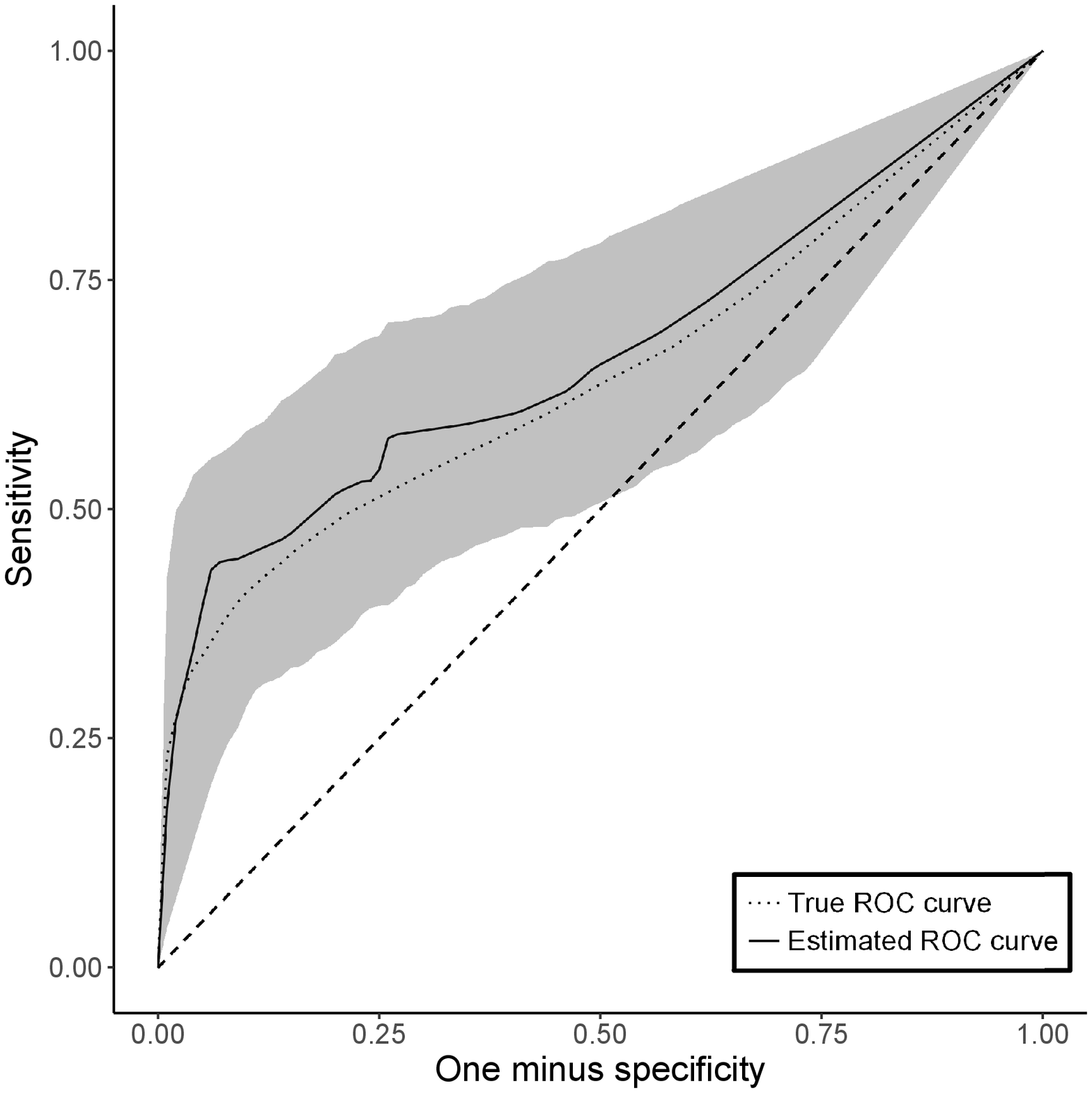}}%
\hfill
\caption{ROC curves and confidence bands for $n = 500$, $p = 2$, and $q = 0.25$, when 
the true model is linear (left) and nonlinear (right).}
\label{roc.cb}
\end{figure}
%
\iffalse
\begin{figure}
\centering
\begin{subfigure}{.5\textwidth}
  \centering 
	\includegraphics[width=.4\linewidth]{roc.lin.quant.2.1.2.1.eps}
  \caption{True model linear.}
  \label{lin.roc.cb}	
\end{subfigure}
\begin{subfigure}{.5\textwidth}
  \centering 
	\includegraphics[width=.4\linewidth]{roc.nonlin.quant.2.1.2.1.eps}
  \caption{True model nonlinear.}
  \label{nonlin.roc.cb}
\end{subfigure}
\caption{ROC curves and confidence bands when $n = 500$, $p = 2$, and $q = 0.25$.}
\label{roc.cb}
\end{figure}
\fi
%
\iffalse
\begin{figure}[h!]
\centering
\includegraphics[scale=0.5]{roc.lin.quant.2.1.2.1.eps}
\caption{ROC curve and confidence bands when true model is linear and $n = 500$, $p = 2$, and $q = 0.25$.}
\label{lin.roc.cb}
\end{figure} 
\begin{figure}[h!]
\centering
\includegraphics[scale=0.5]{roc.nonlin.quant.2.1.2.1.eps}
\caption{ROC curve and confidence bands when true model is nonlinear and $n = 500$, $p = 2$, and $q = 0.25$.}
\label{nonlin.roc.cb}
\end{figure} 
\fi
These figures demonstrate that the proposed quantile bootstrap produces confidence bands that 
capture the true ROC curve and are sufficiently narrow as to provide useful inference about 
the future performance of an estimated SVM classifier. 

\section{Applications to Data} \label{roc.data}

\subsection{Breast Cancer Genomics}

We apply the weighted SVM to the problem of predicting treatment response among patients with breast 
cancer. The full data consist of 323 patients with complete data. 
For each patient, we calculated a collection of 512 gene expression signatures, 
called modules, each of which is a function of patient gene expression data, 
which can be used to predict response to neoadjuvant chemotherapy \citep[][]{fan2011building}. 
We also observe a variety of clinical variables, e.g., age and tumor stage. 
Figure~\ref{FIG:mda} contains ROC curves for predicting response 
to treatment using the linear and Gaussian SVM, logistic regression with LASSO penalty 
\citep[][]{tibshirani1996regression}, and 
random forests \citep[][]{breiman2001random}, along with confidence bands for the linear SVM. 
\begin{figure}[h!]
\centering
\includegraphics[scale = 0.5]{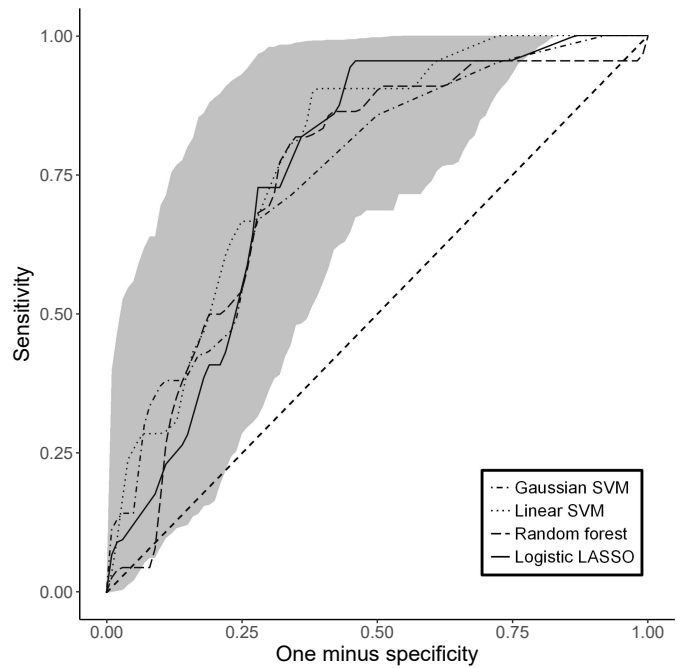}
\caption{ROC curves for predicting response to treatment among breast cancer patients.}
\label{FIG:mda}
\end{figure} 
Each method performs equally well, with each ROC curve falling within the 
confidence bands for the linear SVM.  Table~\ref{mda.tab} contains 
AUC and optimal sensitivity and specificity for each method 
along with the sensitivity and specificity of the unweighted versions 
of each method. 
%\input{mda_table.txt}
% latex table generated in R 3.4.1 by xtable 1.8-2 package
% Wed Nov 15 12:36:56 2017
\begin{table}[h!]
\caption{Comparison of methods applied to breast cancer data.} 
\label{mda.tab}
\centering
\scalebox{0.9}{
\begin{tabular}{c|ccccc}
   \hline
Method & AUC & $\widehat{se}$ (optimal) & $\widehat{sp}$ (optimal) & $\widehat{se}$ (unweighted) & $\widehat{sp}$ (unweighted) \\ 
   \hline
Gaussian SVM & 0.74 & 0.67 & 0.72 & 0.10 & 1.00 \\ 
  Linear SVM & 0.79 & 0.90 & 0.63 & 0.19 & 0.97 \\ 
  Random forest & 0.74 & 0.82 & 0.65 & 0.05 & 0.99 \\ 
  Logistic LASSO & 0.75 & 0.73 & 0.72 & 0.00 & 1.00 \\ 
   \hline
\end{tabular}
}
\end{table}
On these data, the linear SVM achieves the best AUC. Each method achieves a better 
balance between sensitivity and specificity after proper weighting. 
Unweighted classification results in close to perfect specificity 
at the expense of very low sensitivity for each method. 
This is likely due to the imbalance in the data 
(only 22\% of patients in the sample respond). 

\subsection{Diagnosis of Infant Hepatitis C}

We also applied the proposed methods to data from 
the cohort study of mother-to-infant hepatitis C transmission of \cite{shebl2009prospective}. 
In this study, 1863 mother-infant pairs in three Egyptian villages were studied to assess 
risk factors for vertical transmission of hepatitis C virus (HCV). 
Of this sample, 33 infants were positive for both HCV RNA and HCV antibodies at the end of the study. 
We use data from infant follow-up visits 
at 2-4 months and 10-12 months. At each follow-up visit, infants 
were tested for HCV RNA using a polymerase chain reaction (PCR) test and 
HCV antibodies using an enzyme-linked immunosorbent assay (ELISA) test. 
Mothers in the study were also tested for HCV RNA and antibodies during pregnancy. 
In pediatric infectious diseases, it is important to correctly 
diagnose infected infants so that they will be retained in care for subsequent treatment. 
A test with high specificity is also important, as this allows for quickly and reliably 
reassuring families that their child is not infected and needs no further care.  
We use a weighted SVM to estimated a classifier based on the mother's 
test results during pregnancy and infant's test results at 2-4 months. 
While a PCR test at 2-4 months detects HCV viremia, it cannot predict which 
children subsequently become chronically infected, and a PCR test
at 10-12 months remains the gold standard.

In this study, the PCR test achieved a sensitivity of 0.4167 and a specificity of 0.9911. 
The ELISA test achieved a sensitivity of 0.5833 and a specificity of 0.9571. 
Due to a variety of factors, diagnosis during the early months of life 
is difficult. Both PCR and ELISA suffer from low sensitivity 
at 2-4 months for detecting which infants will become chronically infected later. 
It is of interest to see if diagnosis via a weighted SVM can provide even a modest improvement in performance 
thereby reducing the need for a repeat test after 10-12 months of age.
 
We apply the weighted SVM and evaluate performance using 5-fold cross validation. 
Averaging the estimated sensitivity and specificity 
for each value of $\alpha$ over the 5 folds yields the ROC curve 
found in Figure~\ref{FIG:data2}, plotted with bootstrap confidence bands. 
We plot the sensitivity and specificity of the individual 
PCR and ELISA tests as points in the figure. 
\begin{figure}[h!]
\centering
\includegraphics[scale = 0.5]{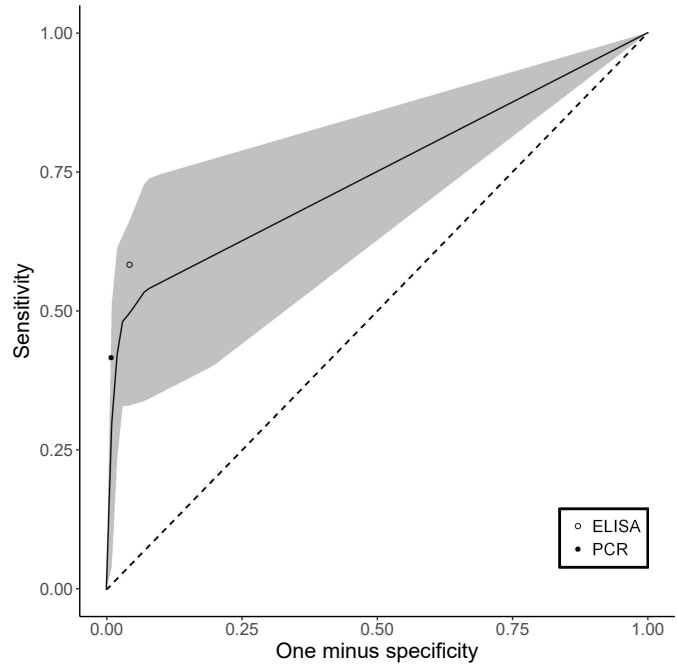}
\caption{ROC curve and confidence band for diagnosis of infant HCV using linear SVM.}
\label{FIG:data2}
\end{figure}
The closest point on the ROC curve to $(0, 1)$ yields an estimated 
sensitivity of 0.6011 and an estimated specificity of 0.8000, which 
provides increased sensitivity and a better balance between sensitivity and specificity 
when compared to the usual diagnostic tests.  
Classification is difficult due to the imbalance of infections and non-infections in the data, 
but a weighted SVM provides increased performance compared to either diagnostic test available. 

\section{Conclusion} \label{roc.conc}

A wide variety of problems in biomedical decision making can be 
expressed as classification problems, such as diagnosing disease 
and predicting response to treatment. In some clinical applications, 
false positives may have very different consequences from false 
negatives; classification methods which can properly weight 
sensitivity and specificity and estimate the optimal ROC curve are needed, 
along with inference methods for the ROC curve. 
Estimating the optimal ROC curve using a weighted SVM has 
been considered by \cite{veropoulos1999controlling}. We have established 
the theoretical justification for estimating the ROC curve with a weighted 
SVM, demonstrated its performance in simulation studies, and provided a 
bootstrap confidence band method for the SVM ROC curve. 

The applications of the weighted SVM in diagnostic medicine are numerous. 
We have demonstrated, for example, that this method can be 
used to improve early infant diagnosis of hepatitis C. 
Early detection of childhood infectious diseases is an important 
public health problem; reliable early diagnosis identifies children 
who could transmit the virus and would benefit from treatment with antivirals. 
We have also demonstrated that the weighted SVM accommodates high
dimensional data and can be used to predict response to neoadjuvant 
breast cancer treatment using genomic information. 

Because machine learning techniques are well suited to binary classification, 
there is great potential for research in applying machine learning to diagnostic 
medicine and other biomedical decision making problems. 
Developing methods of variable selection for the weighted SVM \citep[][]{dasgupta2015feature}
is an important step forward for this research as our simulations indicate that the 
performance of the Gaussian SVM is hindered by noise variables. 
Other areas of future work may include developing methods to accommodate biomarker 
measurements that are taken at different time points from the same patient.  

\section*{Acknowledgements}

The authors gratefully acknowledge the following funding sources: 
NIH T32 CA201159, NIH P01 CA142538, National Center for Advancing Translational Sciences UL1 TR001111, 
NSF DMS-1555141, NSF DMS-1557733, NSF DMS-1513579, NIH 1R01DE024984, NIH U01HD39164, NIH U01AI58372, 
an investigator initiated grant from Merck, NCI Breast SPORE program (P50-CA58223-09A1), 
the Breast Cancer Research Foundation, and the V Foundation for Cancer Research

\section*{Appendix A}

\begin{proof}[Theorem~\ref{se.limit}]
%\begin{proof}[Theorem~1]
\noindent
Note that
\begin{eqnarray*}
\sqrt{n} \left\{ \widehat{se}\left(\widehat{f}_\alpha\right) - se\left(\widehat{f}_\alpha\right) \right\} 
 & = & \sqrt{n} \left[ \frac{\mathbb{E}_n 1\left\{\widehat{f}_\alpha(\bX) > 0\right\}1(A = 1)}{\mathbb{E}_n 1(A = 1)}
 - \frac{\mathbb{E} 1\left\{\widehat{f}_\alpha(\bX) > 0\right\}1(A = 1)}{\mathbb{E} 1(A = 1)}\right] \\
% & = & \sqrt{n} \Bigg[ \frac{(\mathbb{E} - \mathbb{E}_n)1\left\{\widehat{f}_\alpha(\bX) > 0\right\}1(A = 1)}{\mathbb{E}_n 1(A = 1)} \\ 
% &   & - \mathbb{E} 1\left\{\widehat{f}_\alpha(\bX) > 0\right\}1(A = 1) 
% \left\{ \frac{1}{\mathbb{E}1(A = 1)} - \frac{1}{\mathbb{E}_n 1(A = 1)}\right\} \Bigg] \\
 & = & \sqrt{n} \Bigg[ \frac{\mathbb{E}_n 1\left\{\widehat{f}_\alpha(\bX) > 0\right\}1(A = 1)}{\mathbb{E}_n 1(A = 1)} 
 - \mathbb{E}1\left\{ \widehat{f}_\alpha(\bX) > 0\right\} 1(A = 1) \\
 &   & \cdot \left\{ \frac{1}{\mathbb{E}_n 1(A = 1)} + \frac{1}{\mathbb{E}1(A = 1)} - \frac{1}{\mathbb{E}_n 1(A = 1)} \right\} \Bigg] \\
 & = & \sqrt{n} \Bigg[ \frac{\mathbb{E}_n 1\left\{\widehat{f}_\alpha(\bX) > 0\right\}1(A = 1)}{\mathbb{E}_n 1(A = 1)} 
 - \frac{\mathbb{E} 1\left\{\widehat{f}_\alpha(\bX) > 0\right\}1(A = 1)}{\mathbb{E}_n 1(A = 1)} \\
 &   & - \frac{\mathbb{E} 1\left\{\widehat{f}_\alpha(\bX) > 0\right\}1(A = 1)}{\mathbb{E}_n 1(A = 1)} 
 \cdot \left\{ \frac{\mathbb{E}_n 1(A = 1)}{\mathbb{E} 1(A = 1)} - \frac{\mathbb{E} 1(A = 1)}{\mathbb{E} 1(A = 1)} \right\} \Bigg] \\
 & = & \sqrt{n} \left( \frac{(\mathbb{E}_n - \mathbb{E})
 \left[1\left\{\widehat{f}_\alpha(\bX) > 0\right\}1(A = 1)\right]}{\mathbb{E}_n 1(A = 1)} \right) \\
 &   & - \sqrt{n} \left(\frac{\mathbb{E}\left[1\left\{\widehat{f}_\alpha(\bX) > 0\right\}1(A = 1)\right] 
 (\mathbb{E}_n - \mathbb{E})1(A = 1)}{\mathbb{E} 1(A = 1) \mathbb{E}_n 1(A = 1)} \right) \\ 
 & = & \sqrt{n} \left(\mathbb{E}_n - \mathbb{E}\right) \left[ \widehat{\rho}_n^{-1} 1\left\{ \widehat{f}_\alpha(\bX) > 0\right\} 
 1(A = 1) - \widehat{\rho}_n^{-1} se\left(\widehat{f}_\alpha\right) 1(A = 1)\right].
\end{eqnarray*}
In the case of a linear or polynomial decision function, 
$\left\{ \widehat{f}_\alpha : \alpha \in [\delta, 1 - \delta] \right\}$ 
is a Vapnik--Cervonenkis class for any $0 < \delta < 1/2$. 
Thus, $\left\{ 1\left(\widehat{f}_\alpha > 0\right) : \alpha \in [\delta, 1 - \delta] \right\}$ 
is a Donsker class for any $0 < \delta < 1/2$. 
Let $H_1\left(\bX, A; \widehat{f}_\alpha, \widehat{\rho}_n, \alpha\right) = 
 \widehat{\rho}_n^{-1} 1\left\{\widehat{f}_\alpha(\bX) > 0\right\} 1(A = 1) 
 - \widehat{\rho}_n^{-1} se\left(\widehat{f}_\alpha\right) 1(A = 1)$. 
Then, we have that
\begin{equation} \label{H1}
\sqrt{n}\left(\mathbb{E}_n - \mathbb{E}\right) H_1\left(\bX, A; \widehat{f}_\alpha, \widehat{\rho}_n, \alpha\right) 
 = \sqrt{n}\left(\mathbb{E}_n - \mathbb{E}\right) H_1\left(\bX, A; \widetilde{f}_\alpha, \rho, \alpha\right) + o_P(1) 
 \rightsquigarrow \mathbb{G}_1(\alpha),
\end{equation}
where $o_P(1)$ is a quantity converging to 0 in probability uniformly over $\alpha \in [\delta, 1 - \delta]$,  
$H_1$ lies in a Donsker class, 
and $\mathbb{G}_1(\alpha)$ is a mean zero Gaussian process with covariance
\begin{multline*}
\sigma_1(\alpha_1, \alpha_2) = \mathbb{E} \left( \rho^{-2} 1(A = 1) 
\left[1\left\{\widetilde{f}_{\alpha_1}(\bX) > 0\right\} - se\left(\widetilde{f}_{\alpha_1}\right)\right]
\left[1\left\{\widetilde{f}_{\alpha_2}(\bX) > 0\right\} - se\left(\widetilde{f}_{\alpha_2}\right)\right]\right) \\
 - \mathbb{E}\left( \rho^{-1} 1(A = 1) 
\left[1\left\{\widetilde{f}_{\alpha_1}(\bX) > 0\right\} - se\left(\widetilde{f}_{\alpha_1}\right)\right]\right) 
\mathbb{E}\left( \rho^{-1} 1(A = 1) 
\left[1\left\{\widetilde{f}_{\alpha_2}(\bX) > 0\right\} - se\left(\widetilde{f}_{\alpha_2}\right)\right]\right). 
\end{multline*} 
Similarly, for specificity, we have that 
\begin{equation} \label{H2}
\sqrt{n}\left(\mathbb{E}_n - \mathbb{E}\right) H_2\left(\bX, A; \widehat{f}_\alpha, \widehat{\rho}_n, \alpha\right) 
 = \sqrt{n}\left(\mathbb{E}_n - \mathbb{E}\right) H_2\left(\bX, A; \widetilde{f}_\alpha, \rho, \alpha\right) + o_P(1) 
 \rightsquigarrow \mathbb{G}_2(\alpha),
\end{equation}
where $H_2\left(\bX, A; \widehat{f}_\alpha, \widehat{\rho}_n, \alpha\right) = 
 \left(1 - \widehat{\rho}_n\right)^{-1} 1\left\{\widehat{f}_\alpha(\bX) < 0\right\} 1(A = -1) 
 - \left(1 - \widehat{\rho}_n\right)^{-1} sp\left(\widehat{f}_\alpha\right) 1(A = -1)$, 
the $o_P(1)$ is uniform over $\alpha$ as before, $\mathbb{G}_2(\alpha)$ is a mean zero Gaussian process with covariance
\begin{multline*}
\sigma_2(\alpha_1, \alpha_2) = \mathbb{E} \left( \tau^{-2} 1(A = -1) 
\left[1\left\{\widetilde{f}_{\alpha_1}(\bX) < 0\right\} - sp\left(\widetilde{f}_{\alpha_1}\right)\right]
\left[1\left\{\widetilde{f}_{\alpha_2}(\bX) < 0\right\} - sp\left(\widetilde{f}_{\alpha_2}\right)\right]\right) \\
 - \mathbb{E}\left( \tau^{-1} 1(A = -1) 
\left[1\left\{\widetilde{f}_{\alpha_1}(\bX) < 0\right\} - sp\left(\widetilde{f}_{\alpha_1}\right)\right]\right) 
 \mathbb{E}\left( \tau^{-1} 1(A = -1) 
\left[1\left\{\widetilde{f}_{\alpha_2}(\bX) < 0\right\} - sp\left(\widetilde{f}_{\alpha_2}\right)\right]\right), 
\end{multline*}
and $\tau = 1 - \rho$. Now,~(\ref{H1}) and (\ref{H2}) together imply that 
\begin{multline}
\sqrt{n} \left(\mathbb{E}_n - \mathbb{E}\right) \left\{ \begin{array}{c} 
H_1\left(\bX, A; \widehat{f}_\alpha, \widehat{\rho}_n, \alpha\right) \\
H_2\left(\bX, A; \widehat{f}_\alpha, \widehat{\rho}_n, \alpha\right) \end{array} \right\} 
 \\ = \sqrt{n} \left(\mathbb{E}_n - \mathbb{E}\right) \left\{ \begin{array}{c} 
H_1\left(\bX, A; \widetilde{f}_\alpha, \rho, \alpha\right) \\
H_2\left(\bX, A; \widetilde{f}_\alpha, \rho, \alpha\right) \end{array} \right\} + o_P(1) 
\rightsquigarrow \left\{ \begin{array}{c} \mathbb{G}_1(\alpha) \\ \mathbb{G}_2(\alpha) \end{array} \right\},
\end{multline}
because marginal tightness implies joint tightness \citep[see Lemma 7.14 (i) of ][]{kosorok2008introduction}. 
The joint limiting distribution follows. 
\end{proof}

\noindent
\begin{proof}[Theorem~\ref{fisher}]
First, we note that, for each $\bx \in \mathcal{X}$, the optimal classifier is 
\begin{eqnarray*}
D^*_\alpha(\bx) & = & \mathrm{sign}\big[ \mathbb{E}\left\{1(A \ne -1) C_A(\alpha) | \bX = \bx \right\} 
 - \mathbb{E}\left\{1(A \ne 1)C_A(\alpha) | \bX = \bx \right\} \big] \\
 & = & \mathrm{sign} \big\{ \alpha \mathrm{Pr}(A = 1 | \bX = \bx) - (1-\alpha) \mathrm{Pr}(A = -1 | \bX = \bx) \big\}.
\end{eqnarray*}
Next, we note that $\widetilde{f}_\alpha$ minimizes 
\begin{eqnarray*}
 & & \hspace{-0.5in} \mathbb{E}\left[\mathrm{max}\{0, 1 - Af(\bX)\} C_A(\alpha) | \bX = \bx \right] \\
 & = & \mathrm{Pr}(A = 1 | \bX = \bx) \mathbb{E}[\mathrm{max}\{0, 1 - Af(\bX)\} C_A(\alpha) | \bX = \bx, A = 1] \\
 &   & + \mathrm{Pr}(A = -1 | \bX = \bx) \mathbb{E}[\mathrm{max}\{0, 1 - Af(\bX)\} C_A(\alpha) | \bX = \bx, A = -1] \\
 & = & \alpha \mathrm{Pr}(A = 1 | \bX = \bx) \mathrm{max}\{0, 1 - f(\bx)\} \\ 
 &   & + (1 - \alpha) \mathrm{Pr}(A = -1 | \bX = \bx) \mathrm{max}\{0, 1 + f(\bx)\} \\
 & = & f(\bx) \big\{(1-\alpha)\mathrm{Pr}(A = -1 | \bX = \bx) - \alpha \mathrm{Pr}(A = 1 | \bX = \bx) \big\} \\ 
 &   & + \alpha \mathrm{Pr}(A = 1 | \bX = \bx) + (1-\alpha)\mathrm{Pr}(A = -1 | \bX = \bx).
\end{eqnarray*}
We note that as long as both $\alpha$ and $\mathrm{Pr}(A = 1 | \bX = \bx)$ lie in the open interval $(0, 1)$ 
for almost all $\bx$, then $\alpha \mathrm{Pr}(A = 1 | \bX = \bx) \phi\{f(\bX)\} + (1-\alpha) \mathrm{Pr}(A = -1 | \bX = \bx) \phi\{-f(\bX)\}$ 
decreases strictly on $(-\infty, -1]$ and increases strictly on 
$[1, \infty)$. Thus, the minimum $\widetilde{f}_\alpha$ must take values in $[-1, 1]$, which justifies the third equality above. 
We have that $\widetilde{f}_\alpha$ will be positive when $\alpha \mathrm{Pr}(A = 1 | \bX = \bx) > (1-\alpha)\mathrm{Pr}(A = -1 | \bX = \bx)$ and negative otherwise. 
%when $(1-\alpha)\mathrm{Pr}(A = -1 | X = x) > \alpha \mathrm{Pr}(A = 1 | X = x)$. 
The extension to $\alpha \in [0, 1]$ is trivial as long as $0 < \mathrm{Pr}(A = 1 | \bX = \bx) < 1$ for almost all $\bx$. 
Thus, $\widetilde{f}_\alpha(\bx)$ has the same sign as $D^*_\alpha(\bx)$, which completes the proof. 
\end{proof}

\noindent
\begin{proof}[Theorem~\ref{cons}]
Let $\widetilde{f} = \operatorname*{arg\,min}_{f \in \bar{\mathcal{H}}_k} \mathcal{R}_{\alpha, \phi}(f)$.   
Let $\| \cdot \|_k$ be the norm associated with $\mathcal{H}_k$. Note that 
$\widehat{f}_\alpha = \operatorname*{arg\,min}_{f \in \bar{\mathcal{H}}_k} 
\big\{ \mathbb{E}_n L_{\alpha, \phi}(f) + \lambda_n \| f \|^2_k \big\}$, 
where $\mathbb{E}_n$ denotes the empirical measure of $(\bX, A)$. 
We start by finding a bound for $\left\| \sqrt{\lambda_n} \widehat{f}_\alpha\right\|^2_k$. 
By definition of $\widehat{f}_\alpha$, we have that, for any $f \in \mathcal{H}_k$, 
$$
\mathbb{E}_n L_{\alpha, \phi}\left(\widehat{f}_\alpha\right) + \lambda_n \left\| \widehat{f}_\alpha \right\|^2_k 
 \le \mathbb{E}_n L_{\alpha, \phi}(f) + \lambda_n \| f \|^2_k.
$$ 
Setting $f \equiv 0$ in the above and noting that $\| 0 \|^2_k = 0$ and $\phi(0) = 1$ yields 
$$
\mathbb{E}_n L_{\alpha, \phi}\left(\widehat{f}_\alpha\right) + \lambda_n \left\| \widehat{f}_\alpha \right\|^2_k 
 \le \mathbb{E}_n C_A(\alpha). 
$$
We have that $\mathbb{E}_n C_A(\alpha) = n^{-1} \sum_{i=1}^n \left\{\alpha 1(A_i = 1) 
 + (1-\alpha) 1(A_i = -1)\right\} \le \mathrm{max}(\alpha, 1 - \alpha)$ 
and $\mathbb{E}_n L_{\alpha, \phi}\left(\widehat{f}_\alpha\right) \ge 0$. It follows that 
$$
\left\| \sqrt{\lambda_n} \widehat{f}_\alpha \right\|^2_k \le \mathrm{max}(\alpha, 1-\alpha) \le 1
$$

Next, we observe that $\left\{ \sqrt{\lambda_n} \widehat{f}_\alpha 
 : \left\| \sqrt{\lambda_n} \widehat{f}_\alpha \right\|_k \le 1 \right\}$ 
is a unit ball in a RKHS and is contained within a Donsker class. 
By a Donsker preservation result on page 173 of \cite{kosorok2008introduction}, 
$\left\{ \sqrt{\lambda_n} L_{\alpha, \phi}\left(\widehat{f}_\alpha\right) 
 : \left\| \sqrt{\lambda_n} \widehat{f}_\alpha \right\|_k \le 1 \right\}$
is also contained within a Donsker class because $\phi(Af)$ is Lipschitz continuous in $f$. 

The definition of $P$-Donsker gives us that 
$$
\sqrt{n} (\mathbb{E}_n - \mathbb{E}) \sqrt{\lambda_n} L_{\alpha, \phi}\left(\widehat{f}_\alpha\right) = O_P(1).
$$
Thus, 
\begin{eqnarray*}
(\mathbb{E}_n - \mathbb{E})L_{\alpha, \phi}\left(\widehat{f}_\alpha\right) 
 & = & \sqrt{(n\lambda_n)^{-1}} \sqrt{n} (\mathbb{E}_n - \mathbb{E}) \sqrt{\lambda_n} L_{\alpha, \phi}\left(\widehat{f}_\alpha\right) \\
 & = & \sqrt{(n\lambda_n)^{-1}} O_P(1), 
\end{eqnarray*}
which converges to 0 in probability because $n\lambda_n \rightarrow \infty$. 

Next, it follows from the definition of $\widehat{f}_\alpha$ that 
$$
\mathbb{E}_n L_{\alpha, \phi}\left(\widehat{f}_\alpha\right) \le \mathbb{E}_n L_{\alpha, \phi}\left(\widehat{f}_\alpha\right) 
 + \lambda_n \left\| \widehat{f}_\alpha \right\|^2_k 
\le \mathbb{E}_n L_{\alpha, \phi}\left(\widetilde{f}_\alpha\right) + \lambda_n \left\| \widetilde{f}_\alpha \right\|^2_k. 
$$
Taking the $\operatorname*{lim\,sup}_n$ on both sides and using the fact that $\lambda_n \rightarrow 0$ yields 
$$
\operatorname*{lim\,sup}_n \mathbb{E}_n L_{\alpha, \phi}\left(\widehat{f}_\alpha\right) 
 \le \mathbb{E}\left\{L_{\alpha, \phi}\left(\widetilde{f}_\alpha\right)\right\}, 
$$
almost surely. 
%where $E$ denotes expectation with respect to the distribution $P$. 
%It follows that 
%$$
%\operatorname*{lim\,sup}_n \mathbb{E}_n L_{\alpha, \phi}(\hat{f}_\alpha) \le \inf_{f \in \mathcal{H}_k} E\left[L_{\alpha, \phi}(f)\right]
%$$ 
%almost surely. 
Thus, for all $n$ large enough, we have 
$$
\mathbb{E}_n L_{\alpha, \phi}\left(\widehat{f}_\alpha\right) 
 \le \mathbb{E}\left\{ L_{\alpha, \phi} \left(\widetilde{f}_\alpha\right)\right\} 
 \le \mathbb{E}\left\{L_{\alpha, \phi} \left(\widehat{f}_\alpha\right)\right\}
$$
almost surely. For $n$ large enough, we have 
$\left|\mathcal{R}_{\alpha, \phi}\left(\widehat{f}_\alpha\right) 
 - \mathcal{R}_{\alpha, \phi} \left(\widetilde{f}_\alpha\right)\right|
 \le \left|\left(\mathbb{E}_n - \mathbb{E}\right) L_{\alpha, \phi}\left(\widehat{f}_\alpha\right)\right| = o_P(1)$. 
%Fix $\epsilon > 0$ and $\delta > 0$. Because $(\mathbb{E}_n - E) L_{\alpha, \phi}(\hat{f}_\alpha) = o_P(1)$ as shown above, 
%we have that for $n$ large enough, $\mathrm{Pr}(|(\mathbb{E}_n - E) L_{\alpha, \phi}(\hat{f}_\alpha)| < \epsilon) > 1 - \delta$. 
%Thus, for $n$ large enough, 
%$$
%\mathrm{Pr}(|\mathcal{R}_{\alpha, \phi}(\hat{f}_\alpha) - \mathcal{R}_{\alpha, \phi} (\tilde{f}_\alpha)| < \epsilon) > 1 - \delta. 
%$$
By Lemma~\ref{excessrisk}, we have that 
$\left|\mathcal{R}_{\alpha}\left(\widehat{f}_\alpha\right) - \inf_f \mathcal{R}_{\alpha} (f)\right| \le 
 \left|\mathcal{R}_{\alpha, \phi}\left(\widehat{f}_\alpha\right) - \inf_f \mathcal{R}_{\alpha, \phi} (f)\right| = o_P(1)$.
\end{proof}

\noindent
\begin{proof}[Lemma~\ref{gc}.] 
%\begin{proof}[Lemma~2.] 
When the decision function is linear, quadratic, or polynomial, 
it lies in a Vapnik--Cervonenkis (VC) class 
as on page 238 of \cite{friedman2001elements}. This implies that 
$\left\{\widehat{f}_\alpha : \alpha \in [0, 1]\right\}$ 
is contained in a GC class by Theorem 9.3 of Kosorok (2008). 

Next, we consider the case where the decision function is estimated using a Gaussian kernel. 
Because the exponential function is monotone, the class of all functions of the form 
$f(\bx; \mathbf{y}) = \exp( -c \| \bx - \mathbf{y} \|^2 )$ is a 
VC class by Lemma 9.9 (viii) of \cite{kosorok2008introduction}. 
The RKHS is a VC-hull class as defined on page 158 of \cite{kosorok2008introduction}. 
It now follows that $\left\{\widehat{f}_\alpha : \alpha \in [0, 1]\right\}$ 
is contained in a GC class by Corollary 9.5 of \cite{kosorok2008introduction}. 
\end{proof}

\noindent
\begin{proof}[Theorem \ref{unifcons}.] 
%\begin{proof}[Theorem~4.] 
%Lemma~\ref{excessrisk}
From Lemma~\ref{excessrisk}, the claim will follow if we can show that 
$$
\operatorname*{sup}_{\alpha \in [0, 1]} \left| \mathcal{R}_{\alpha, \phi} \left(\widehat{f}_\alpha\right) 
 - \inf_{f \in \bar{\mathcal{H}}_k} \mathcal{R}_{\alpha, \phi}(f) \right| \xrightarrow[]{P} 0. 
$$
Following the proof of Theorem~3.3 of \cite{zhao2012estimating},  
we have that, for all $n$ large enough, 
$$
\mathbb{E}_n L_{\alpha, \phi} \left(\widehat{f}_\alpha\right) \le 
\mathbb{E}\left\{ L_{\alpha, \phi} \left(\widetilde{f}_\alpha\right)\right\} \le 
\mathbb{E}\left\{L_{\alpha, \phi} \left(\widehat{f}_\alpha\right)\right\} 
$$
almost surely, where $L_{\alpha, \phi}(f)$ is as defined in Section~\ref{roc.theory}. %\ref{roc.theory}
Thus, for all $n$ large enough, we have $\left|\mathcal{R}_{\alpha, \phi}\left(\widehat{f}_\alpha\right) 
 - \mathcal{R}_{\alpha, \phi} \left(\widetilde{f}_\alpha\right)\right|
\le \left|(\mathbb{E}_n - \mathbb{E}) L_{\alpha, \phi}\left(\widehat{f}_\alpha\right)\right|$ almost surely. 
It follows that, for $n$ large enough, 
\begin{eqnarray*}
\sup_{\alpha \in [0, 1]} \left|\mathcal{R}_{\alpha, \phi}\left(\widehat{f}_\alpha\right) - 
\mathcal{R}_{\alpha, \phi} \left(\widetilde{f}_\alpha\right) \right|
\le \sup_{\alpha \in [0, 1]} \left|(\mathbb{E}_n - \mathbb{E}) 
L_{\alpha, \phi}\left(\widehat{f}_\alpha\right) \right| \xrightarrow[]{P} 0, 
\end{eqnarray*}
where the convergence follows from Lemma~\ref{gc} %\ref{gc} 
and the arguments in Section~\ref{roc.theory}. %\ref{roc.theory}. 
\end{proof}

\noindent
\begin{proof}[Lemma \ref{continuity}.] 
%\begin{proof}[Lemma~3.]
Following Remark 4, $D^*_\alpha(\bX) = \mathrm{sign}\left\{ f^*_\alpha(\bX)\right\}$, where 
$$
f^*_\alpha(\bX) = \frac{\mathrm{Pr}(A = 1 | \bX)}{1 - \mathrm{Pr}(A = 1 | \bX)} 
 - \frac{1 -\alpha}{\alpha}. 
$$
Thus, the sensitivity and specificity of $D^*_\alpha$ are 
$$
se^*_\alpha = \mathrm{Pr} \left\{ \frac{\mathrm{Pr}(A = 1 | \bX)}{1 - \mathrm{Pr}(A = 1 | \bX)} 
 - \frac{1 -\alpha}{\alpha} > 0 \big| A = 1 \right\} 
$$
and 
$$
sp^*_\alpha = \mathrm{Pr} \left\{ \frac{\mathrm{Pr}(A = 1 | \bX)}{1 - \mathrm{Pr}(A = 1 | \bX)} 
 - \frac{1 -\alpha}{\alpha} < 0 \big| A = -1 \right\}, 
$$
which are continuous in $\alpha$ when $\mathrm{Pr}(A = 1 | \bX)$ 
is a continuous random variable. 
\end{proof}

\begin{proof}[Theorem \ref{risk.continuity}.]
%\begin{proof}[Theorem~5.]
Let $\alpha_n$ be a sequence such that $\alpha_n \in [0, 1]$ for $n \ge 1$ 
and $\alpha_n \rightarrow \alpha$. Assume that 
$\mathrm{lim \, sup}_{n \rightarrow \infty} \mathcal{R}_{\alpha_n}\left(\widetilde{f}_{\alpha_n}\right)
 > \mathcal{R}_\alpha\left(\widetilde{f}_\alpha\right)$.  
Because $\widetilde{f}_\alpha \in \mathcal{F}$, we have by the definition of 
$\widetilde{f}_{\alpha_n}$ that $\mathcal{R}_{\alpha_n}\left(\widetilde{f}_{\alpha_n}\right) 
 \le \mathcal{R}_{\alpha_n}\left(\widetilde{f}_\alpha\right)$. 
However, by continuity of $\mathcal{R}_\alpha(f)$ 
for fixed $f$, we have that $\mathcal{R}_{\alpha_n}\left(\widetilde{f}_\alpha\right) 
 \rightarrow \mathcal{R}_\alpha\left(\widetilde{f}_\alpha\right)$, a contradiction. 
Next, assume that 
$\mathrm{lim \, inf}_{n \rightarrow \infty} \mathcal{R}_{\alpha_n} \left(\widetilde{f}_{\alpha_n}\right) 
 < \mathcal{R}_\alpha\left(\widetilde{f}_\alpha\right)$. However, we have that 
$$
\mathrm{lim \, inf}_{n \rightarrow \infty} \mathcal{R}_{\alpha_n} \left(\widetilde{f}_{\alpha_n}\right) 
 = \mathrm{lim \, inf}_{n \rightarrow \infty} \mathcal{R}_\alpha \left(\widetilde{f}_{\alpha_n}\right) 
 \ge \mathcal{R}_\alpha\left(\widetilde{f}_\alpha\right), 
$$
by continuity of $\mathcal{R}_\alpha(f)$ for fixed $f$, the definition of 
$\widetilde{f}_\alpha$ and the fact that $\widetilde{f}_{\alpha_n} \in \mathcal{F}$. 
This yields a contradiction and we have that $\mathcal{R}_{\alpha_n}\left(\widetilde{f}_{\alpha_n}\right)
 \rightarrow \mathcal{R}_\alpha\left(\widetilde{f}_\alpha\right)$, which 
completes the proof. 
\end{proof}

\begin{proof}[Corollary \ref{cor1}.]
%\begin{proof}[Corollary~2.]
%\ref{gc}
By arguments in the proof of Lemma~\ref{gc}, $\mathcal{F}$ is 
a GC class. By arguments in Section~\ref{roc.theory}, %\ref{roc.theory}, 
$\left\{ L_{\alpha, \phi}(f) : f \in \mathcal{F} \right\}$ is a GC class. 
Thus, $\sup_{f \in \mathcal{F}} \left| (\mathbb{E}_n - \mathbb{E}) L_{\alpha, \phi}(f) \right| \xrightarrow[]{P} 0$. 
Define $\widetilde{f}_{\alpha, \phi} = \operatorname*{arg \, min}_{f \in \mathcal{F}} \mathbb{E} L_{\alpha, \phi}(f)$. 
Because $\mathbb{E} L_{\alpha, \phi}(f)$ is convex in $f$ and 
$\mathcal{F}$ is a convex set, $\widetilde{f}_{\alpha, \phi}$ is 
a unique minimizer. Because $\mathbb{E} L_{\alpha, \phi}(f)$ is continuous in $f$, 
the necessary identifiability condition holds by Lemma~14.3 of \cite{kosorok2008introduction} 
and, by Theorem~2.12 of \cite{kosorok2008introduction}, 
$\sup_{x \in \mathcal{X}} \left| \widehat{f}_\alpha(x) - \widetilde{f}_{\alpha, \phi}(x) \right| = o_P(1)$. 
It follows that $\left| se\left(\widehat{f}_\alpha\right) - se\left(\widetilde{f}_{\alpha, \phi}\right) \right| = o_P(1)$ 
and $\left| sp\left(\widehat{f}_\alpha\right) - sp\left(\widetilde{f}_{\alpha, \phi}\right) \right| = o_P(1)$. 
Define $\widetilde{se}_\alpha = se\left(\widetilde{f}_{\alpha, \phi}\right)$ 
and $\widetilde{sp}_\alpha = sp\left(\widetilde{f}_{\alpha, \phi}\right)$ and note that 
\begin{multline*}
\left| \alpha \rho \left\{ se\left( \widehat{f}_\alpha\right) - \widetilde{se}_\alpha \right\}
 + (1 - \alpha) (1 - \rho) \left\{ sp\left( \widehat{f}_\alpha\right) - \widetilde{sp}_\alpha \right\} \right|
\\ \le \left| se\left(\widehat{f}_\alpha\right) - \widetilde{se}_\alpha \right| 
+ \left| sp\left(\widehat{f}_\alpha\right) - \widetilde{sp}_\alpha \right| = o_P(1).
\end{multline*}
%\ref{cons} \ref{true.opt.rmrk}
By Theorem~\ref{cons} and Remark~\ref{true.opt.rmrk}, it must hold that 
$\alpha \rho \widetilde{se}_\alpha + (1 - \alpha)(1 - \rho) \widetilde{sp}_\alpha 
 = \alpha \rho se\left(\widetilde{f}_\alpha\right) + (1 - \alpha)(1 - \rho) sp\left(\widetilde{f}_\alpha \right)$. 
%\ref{unifcons}
Finally, by Theorem~\ref{unifcons}, it must hold that 
$\mathrm{sup}_{\alpha \in [0, 1]} \left| se\left(\widehat{f}_\alpha\right)
- \widetilde{se}_\alpha \right| \xrightarrow[]{P} 0$
and 
$\mathrm{sup}_{\alpha \in [0, 1]} \left| sp\left(\widehat{f}_\alpha\right)
- \widetilde{sp}_\alpha \right| \xrightarrow[]{P} 0$. 
\iffalse
Remark~\ref{true.opt.rmrk} and Theorem~\ref{unifcons} give use that 
$$
\operatorname*{sup}_{\alpha \in [0, 1]} \left| \rho \alpha 
 \left\{ se\left(\widetilde{f}_\alpha\right) - se\left(\widehat{f}_\alpha\right)\right\} 
 - (1 - \rho) (1 - \alpha) \left\{ sp\left(\widetilde{f}_\alpha\right) - sp\left(\widehat{f}_\alpha\right)\right\} 
\right| \xrightarrow[]{P} 0
$$
for all $\rho \in (0, 1)$. Because $\rho$ is arbitrary, both
$\mathrm{sup}_{\alpha \in [0, 1]} \left| \alpha 
\left\{ se\left(\widetilde{f}_\alpha\right) - se\left(\widehat{f}_\alpha\right)\right\} \right|$, 
and 
$\mathrm{sup}_{\alpha \in [0, 1]} \left| (1 - \alpha) 
\left\{ sp\left(\widetilde{f}_\alpha\right) - sp\left(\widehat{f}_\alpha\right)\right\} \right|$ 
converge to 0 in probability, and the result follows. 
\fi
\end{proof}

\begin{proof}[Corollary \ref{cor2}.]
%\begin{proof}[Corollary~3.]
First note that 
\begin{eqnarray*}
\operatorname*{sup}_{\alpha \in [0, 1]} \left| \widehat{se}\left(\widehat{f}_\alpha\right)
 - \widetilde{se}_\alpha\right| 
 & \le & \operatorname*{sup}_{\alpha \in [0, 1]} \left| \widehat{se}\left(\widehat{f}_\alpha\right) 
 - se\left(\widehat{f}_\alpha\right) \right| + 
 \operatorname*{sup}_{\alpha \in [0, 1]} \left| se\left(\widehat{f}_\alpha\right) - \widetilde{se}_\alpha \right| \\
 & \le & \operatorname*{sup}_{\alpha \in [0, 1]} \left| \frac{\mathbb{E}_n 1\left\{ \widehat{f}_\alpha(\bX) > 0, A = 1\right\}}
{\mathbb{E}_n 1(A = 1)} - \frac{\mathbb{E}_n 1\left\{ \widehat{f}_\alpha(\bX) > 0, A = 1\right\}}{\rho} \right| \\
 & & + \operatorname*{sup}_{\alpha \in [0, 1]} \left| \frac{\mathbb{E}_n 1\left\{ \widehat{f}_\alpha(\bX) > 0, A = 1\right\}}{\rho} 
 - \frac{\mathbb{E} 1\left\{ \widehat{f}_\alpha(\bX) > 0, A = 1\right\}}{\rho} \right| \\
 & & + \operatorname*{sup}_{\alpha \in [0, 1]} \left| se\left(\widehat{f}_\alpha\right) - \widetilde{se}_\alpha \right|.  
\end{eqnarray*}
The first piece above is equal to 
\begin{multline*}
\operatorname*{sup}_{\alpha \in [0, 1]} \left| \left[ \mathbb{E}_n 
 1\left\{ \widehat{f}_\alpha(\bX) > 0, A = 1\right\} \right] 
 \left[ \left\{\mathbb{E}_n 1(A = 1)\right\}^{-1} - \rho^{-1} \right] \right| \\
 \le \left[ \operatorname*{sup}_{\alpha \in [0, 1]} \left| 
 \mathbb{E}_n 1\left\{\widehat{f}_\alpha(\bX) > 0, A = 1\right\} \right| \right] 
 \left| \left\{\mathbb{E}_n 1(A = 1)\right\}^{-1} - \rho^{-1} \right| 
 = O_P(1) o_P(1), 
\end{multline*}
where $\left| \left\{\mathbb{E}_n 1(A = 1)\right\}^{-1} - \rho^{-1} \right| = o_P(1)$ 
by the continuous mapping theorem and the assumption that $\rho > 0$. 

Next, let $\epsilon > 0$ and $g_\epsilon(x) = \epsilon^{-1}\left(x^+ \wedge \epsilon\right)$ 
where $x^+ = \max(x, 0)$ and $\wedge$ denotes minimum. 
Note that, by Lemma~\ref{ep.lem} 
below, 
$
\sup_{\alpha \in [0, 1]} \left| (\mathbb{E}_n - \mathbb{E}) \rho^{-1} 
 g_\epsilon\left\{ \widehat{f}_\alpha(\bX)\right\} 1(A = 1)\right| = o_P(1),
$
and thus, 
$
\sup_{\alpha \in [0, 1]} \left| (\mathbb{E}_n - \mathbb{E}) \rho^{-1} 
 1\left\{ \widehat{f}_\alpha(\bX) > 0\right\} 1(A = 1)\right| \le o_P(1) + 2\epsilon. 
$
Since $\epsilon$ was arbitrary, the second piece above is equal to $o_P(1)$. 
%Next, we have that $\left\{ \rho^{-1} 1\left\{\widehat{f}_\alpha(\bX) > 0\right\} 
%1(A = 1) : \alpha \in [0, 1] \right\}$ is a GC class by Lemma~\ref{gc} above and 
%Corollary~9.26 parts (iii) and (ii) of \cite{kosorok2008introduction}. Thus, 
%the second piece above is equal to $o_P(1)$. 
The third piece above is equal to $o_P(1)$ 
by Corollary~\ref{cor1} 
above. The proof for specificity is analogous and is omitted. 
\end{proof}

\begin{lem} \label{ep.lem}
Let $\mathcal{F}$ be a GC class and $g: \mathbb{R} \rightarrow \mathbb{R}$ 
be a continuous, bounded function such that 
$\sup_{x \in \mathbb{R}} |g(x)| = c_0$, 
$\lim_{x \rightarrow -\infty} g(x) = c_1$, and 
$\lim_{x \rightarrow \infty} g(x) = c_2$. Then, 
$g(\mathcal{F}) = \left\{ g(f) : f \in \mathcal{F}\right\}$ is a GC class. 
\end{lem}

\begin{proof}
By Lemma 8.13 of \cite{kosorok2008introduction}, $\mathbb{E}\|f - \mathbb{E}f\|^*_\mathcal{F} < \infty$. 
Fix $\epsilon > 0$ and find $M < \infty$ and $-\infty < k_1 < 0 < k_2 < \infty$ such that 
$\mathbb{E} \left\{ 1\left( \|f - \mathbb{E}f\|^*_\mathcal{F} > M\right) \right\} 
 \le \epsilon / 2c_0$ and $\sup_{x \le k_1} |g(x) - c_1| \le \epsilon$ 
and $\sup_{x \ge k_2} |g(x) - c_2| \le \epsilon$. 
Set $b_1 = k_1 - 2M$ and $b_2 = k_2 + 2M$. Then, 
$g(f) = g(\mathbb{E}f + f - \mathbb{E}f) 1\left(\|f - \mathbb{E}f\|^*_\mathcal{F} \le M\right) 
 + g(f) 1\left(\|f - \mathbb{E}f\|^*_\mathcal{F} > M\right)$, and 
$$
\|(\mathbb{E}_n - \mathbb{E})g(f) \|_\mathcal{F} \le \|(\mathbb{E}_n - \mathbb{E})g(f)\|_\mathcal{F} 
1\left(\|f - \mathbb{E}f\|^*_\mathcal{F} \le M\right)  + 2c_0 1\left(\|f - \mathbb{E}f\|^*_\mathcal{F} > M\right). 
$$
We have that 
\begin{eqnarray*}
\|(\mathbb{E}_n - \mathbb{E}) g(f)\|_\mathcal{F} 1\left(\|f - \mathbb{E}f\|^*_\mathcal{F} \le M\right) 
 & = & \|(\mathbb{E}_n - \mathbb{E})g(\mathbb{E}f + \dot{f})\|_\mathcal{F} 
1\left(\|f - \mathbb{E}f\|^*_\mathcal{F} \le M\right) \\
 & \le & \|(\mathbb{E}_n - \mathbb{E}) g(f) \|_{\dot{\mathcal{F}}_\epsilon} 
1\left(\|f - \mathbb{E}f\|^*_\mathcal{F} \le M\right) + 2\epsilon, 
\end{eqnarray*}
where $\dot{f} = f - \mathbb{E}f$ and $\dot{\mathcal{F}}_\epsilon 
 = \left\{c + \dot{f} : f \in \mathcal{F}, b_1 \le c \le b_2\right\}$. 
Since $g$ is continuous and $C_\epsilon = \left\{x : b_1 \le x \le b_2 \right\}$ 
is compact, there exists a $\delta > 0$ such that 
$\sup_{x_1, x_2 \in C_\epsilon : |x_1 - x_2| \le \delta} |g(x_1) - g(x_2)| \le \epsilon$. 
Let $D_\epsilon$ be a finite subset of $[b_1, b_2]$ such that 
$\sup_{c \in [b_1, b_2]} \inf_{\widetilde{c} \in D_\epsilon} \left|c - \widetilde{c}\right| \le \delta$. 
Now, define $\dot{\mathcal{F}}_\epsilon^* = \left\{ c + \dot{f} : f \in \mathcal{F}, c \in D_\epsilon\right\}$. 
Let $\widetilde{c}(c) = \mathrm{arg \, inf}_{\widetilde{c} \in D_\epsilon} |c - \widetilde{c}|$ 
and note that, provided $\|f - \mathbb{E}f\|^*_\mathcal{F} \le M$, 
\begin{eqnarray*}
\sup_{f_1 \in \dot{\mathcal{F}}_\epsilon} \inf_{f_2 \in \dot{\mathcal{F}}^*_\epsilon} \|g(f_1) - g(f_2)\|
& \le & \sup_{f \in \mathcal{F}, b_1 \le c \le b_2} \left|g\left(c + \dot{f}\right) 
 - g\left\{\widetilde{c}(c) + \dot{f}\right\} \right| \\
 & \le & \sup_{x_1, x_2 \in C_\epsilon : |x_1 - x_2| \le \delta} |g(x_1) - g(x_2)| \\
 & \le & \epsilon.  
\end{eqnarray*}
This implies that $\|(\mathbb{E}_n - \mathbb{E}) g(f)\|_{\dot{\mathcal{F}}_\epsilon} \le 
\|(\mathbb{E}_n - \mathbb{E}) g(f)\|_{\dot{\mathcal{F}}_\epsilon^*} + 2\epsilon$ 
and $\|(\mathbb{E}_n - \mathbb{E}) g(f)\|_{\dot{\mathcal{F}}_\epsilon^*} 
 \le \max_{c \in D_\epsilon} \|(\mathbb{E}_n - \mathbb{E}) g(c + f)\|_{\dot{\mathcal{F}}}$, 
where $\dot{\mathcal{F}} = \left\{ f - \mathbb{E}f : f \in \mathcal{F}\right\}$. 
We have that $\|(\mathbb{E}_n - \mathbb{E}) g(c + f)\|_{\dot{\mathcal{F}}} 
1\left( \|f - \mathbb{E}f\|^*_\mathcal{F} \le M \right) \le \|(\mathbb{E}_n - \mathbb{E})g(c + f)\|_{\dot{\mathcal{F}}}$. 
Theorem~9.26 of \cite{kosorok2008introduction} gives us that $\left\{ g(c + f) : f \in \dot{\mathcal{F}}\right\}$ 
is GC, since the mapping $x \mapsto g(c + x)$ is continuous and bounded and
$\dot{\mathcal{F}}$ has an integrable envelope. Thus, 
$\|(\mathbb{E}_n - \mathbb{E})g(f)\|_\mathcal{F} \le o_P(1) + 4\epsilon 
 + 2 c_0 1\left(\|f - \mathbb{E}f\|_\mathcal{F} > M\right)$. 
When $\mathrm{Pr}\left(\|f - \mathbb{E}f\|_\mathcal{F} > M\right) \le \epsilon$ and 
$\epsilon$ is arbitrary, $\|(\mathbb{E}_n - \mathbb{E}) g(f) \|_\mathcal{F} = o_P(1)$. 
Combining this with Lemma~8.16 of \cite{kosorok2008introduction}, 
we obtain the desired convergence. 
\end{proof}

\section*{Appendix B}

Table~\ref{sim1.sesp.nonlin} 
below contains optimal sensitivities 
and specificities averaged across replications for the four 
methods when the true generative model is nonlinear. 
%\input{roc.sim1.sesp.nonlin.tables.txt}
% latex table generated in R 3.2.2 by xtable 1.7-4 package
% Tue Apr 18 17:50:09 2017
\begin{table}[h!]
\caption{Average optimal sensitivity and specificity when true model is nonlinear.} 
\label{sim1.sesp.nonlin}
\centering
\scalebox{0.9}{
\begin{tabular}{ccc|cccccccc}
  \hline
   &  &  & \multicolumn{2}{c}{Linear SVM} & \multicolumn{2}{c}{Gaussian SVM} & \multicolumn{2}{c}{Logistic} & \multicolumn{2}{c}{Semiparametric} \\$n$ & $p$ & $q$ & $se$ & $sp$ & $se$ & $sp$ & $se$ & $sp$ & $se$ & $sp$ \\ 
   \hline
250 & 2 & 0.05 & 0.64 & 0.59 & 0.70 & 0.77 & 0.61 & 0.58 & 0.53 & 0.58 \\ 
   &  & 0.25 & 0.58 & 0.69 & 0.73 & 0.80 & 0.59 & 0.66 & 0.56 & 0.61 \\ 
   & 5 & 0.05 & 0.69 & 0.68 & 0.71 & 0.72 & 0.68 & 0.70 & 0.51 & 0.58 \\ 
   &  & 0.25 & 0.69 & 0.73 & 0.72 & 0.75 & 0.70 & 0.73 & 0.56 & 0.62 \\ 
   & 10 & 0.05 & 0.67 & 0.70 & 0.56 & 0.55 & 0.67 & 0.69 & 0.52 & 0.59 \\ 
   &  & 0.25 & 0.68 & 0.74 & 0.68 & 0.44 & 0.69 & 0.72 & 0.56 & 0.62 \\ 
   \hline
500 & 2 & 0.05 & 0.64 & 0.55 & 0.71 & 0.78 & 0.64 & 0.54 & 0.53 & 0.58 \\ 
   &  & 0.25 & 0.60 & 0.66 & 0.71 & 0.78 & 0.60 & 0.62 & 0.56 & 0.61 \\ 
   & 5 & 0.05 & 0.68 & 0.68 & 0.71 & 0.75 & 0.66 & 0.69 & 0.51 & 0.59 \\ 
   &  & 0.25 & 0.69 & 0.73 & 0.73 & 0.78 & 0.69 & 0.72 & 0.56 & 0.62 \\ 
   & 10 & 0.05 & 0.66 & 0.67 & 0.63 & 0.53 & 0.67 & 0.67 & 0.51 & 0.58 \\ 
   &  & 0.25 & 0.68 & 0.73 & 0.72 & 0.45 & 0.69 & 0.70 & 0.56 & 0.62 \\ 
   \hline
\end{tabular}
}
\end{table}

Table~\ref{sim1.sesp.nonlin.uw} below contains estimated sensitivities and specificities 
of an unweighted SVM when the true model is nonlinear.
%\input{roc.sim1.sesp.nonlin.uw.tables.txt}
% latex table generated in R 3.2.2 by xtable 1.7-4 package
% Tue Apr 18 17:50:09 2017
\begin{table}[h!]
\caption{Average sensitivity and specificity of unweighted SVM when true model is nonlinear.} 
\label{sim1.sesp.nonlin.uw}
\centering
\begin{tabular}{ccc|cccc}
  \hline
   & & & \multicolumn{2}{c}{Linear SVM} & \multicolumn{2}{c}{Gaussian SVM} \\$n$ & $p$ & $q$ & $se$ & $sp$ & $se$ & $sp$ \\ 
   \hline
250 & 2 & 0.05 & 0.91 & 0.16 & 0.74 & 0.65 \\ 
   &  & 0.25 & 0.88 & 0.20 & 0.77 & 0.65 \\ 
   & 5 & 0.05 & 0.58 & 0.72 & 0.64 & 0.70 \\ 
   &  & 0.25 & 0.70 & 0.64 & 0.76 & 0.61 \\ 
   & 10 & 0.05 & 0.58 & 0.71 & 0.18 & 0.85 \\ 
   &  & 0.25 & 0.69 & 0.64 & 0.81 & 0.24 \\ 
   \hline
500 & 2 & 0.05 & 0.93 & 0.14 & 0.74 & 0.69 \\ 
   &  & 0.25 & 0.94 & 0.10 & 0.76 & 0.66 \\ 
   & 5 & 0.05 & 0.60 & 0.71 & 0.68 & 0.72 \\ 
   &  & 0.25 & 0.71 & 0.67 & 0.77 & 0.66 \\ 
   & 10 & 0.05 & 0.59 & 0.70 & 0.25 & 0.82 \\ 
   &  & 0.25 & 0.70 & 0.65 & 0.85 & 0.24 \\ 
   \hline
\end{tabular}
\end{table}

Table~\ref{sim1.auc.lin} 
below contains estimated AUC's 
averaged across replications and Monte Carlo standard 
deviations of AUC's for the four methods 
when the true generative model is linear. 
%\input{roc.sim1.auc.lin.tables.txt}
% latex table generated in R 3.2.2 by xtable 1.7-4 package
% Tue Apr 18 17:50:09 2017
\begin{table}[h!]
\caption{Average AUC when true model is linear.} 
\label{sim1.auc.lin}
\centering
\begin{tabular}{ccc|cccc}
   \hline
$n$ & $p$ & $q$ & Linear SVM & Gaussian SVM & Logistic & Semiparametric \\ 
   \hline
250 & 2 & 0.05 & 0.83 (0.05) & 0.80 (0.05) & 0.82 (0.05) & 0.80 (0.03) \\ 
   &  & 0.25 & 0.84 (0.04) & 0.82 (0.06) & 0.84 (0.05) & 0.81 (0.03) \\ 
   & 5 & 0.05 & 0.87 (0.05) & 0.75 (0.06) & 0.87 (0.05) & 0.77 (0.03) \\ 
   &  & 0.25 & 0.89 (0.04) & 0.79 (0.05) & 0.89 (0.03) & 0.78 (0.03) \\ 
   & 10 & 0.05 & 0.85 (0.05) & 0.53 (0.03) & 0.86 (0.04) & 0.77 (0.03) \\ 
   &  & 0.25 & 0.87 (0.05) & 0.54 (0.04) & 0.88 (0.04) & 0.78 (0.03) \\ 
   \hline
500 & 2 & 0.05 & 0.83 (0.04) & 0.82 (0.04) & 0.83 (0.04) & 0.80 (0.02) \\ 
   &  & 0.25 & 0.85 (0.03) & 0.83 (0.03) & 0.84 (0.03) & 0.81 (0.02) \\ 
   & 5 & 0.05 & 0.88 (0.03) & 0.79 (0.04) & 0.88 (0.03) & 0.77 (0.02) \\ 
   &  & 0.25 & 0.90 (0.03) & 0.82 (0.04) & 0.89 (0.03) & 0.78 (0.02) \\ 
   & 10 & 0.05 & 0.86 (0.03) & 0.55 (0.04) & 0.88 (0.03) & 0.77 (0.02) \\ 
   &  & 0.25 & 0.89 (0.03) & 0.59 (0.05) & 0.89 (0.03) & 0.78 (0.02) \\ 
   \hline
\end{tabular}
\end{table}
The linear SVM and logistic regression perform the best 
across $n$, $p$, and $q$. The Gaussian SVM performs poorly 
when $p = 10$ due to the presence of noise variables and 
the semiparametric ROC curve performs poorly when $p > 2$ 
because it only uses a single component of $\bX$. 

Table~\ref{sim1.sesp.lin} 
contains optimal sensitivities 
and specificities averaged across replications for the four 
methods when the true generative model is linear. 
%\input{roc.sim1.sesp.lin.tables.txt}
% latex table generated in R 3.2.2 by xtable 1.7-4 package
% Tue Apr 18 17:50:09 2017
\begin{table}[h!]
\caption{Average optimal sensitivity and specificity when true model is linear.} 
\label{sim1.sesp.lin}
\centering
\scalebox{0.9}{
\begin{tabular}{ccc|cccccccc}
  \hline
   &  &  & \multicolumn{2}{c}{Linear SVM} & \multicolumn{2}{c}{Gaussian SVM} & \multicolumn{2}{c}{Logistic} & \multicolumn{2}{c}{Semiparametric} \\$n$ & $p$ & $q$ & $se$ & $sp$ & $se$ & $sp$ & $se$ & $sp$ & $se$ & $sp$ \\ 
   \hline
250 & 2 & 0.05 & 0.78 & 0.77 & 0.76 & 0.76 & 0.78 & 0.77 & 0.73 & 0.72 \\ 
   &  & 0.25 & 0.80 & 0.78 & 0.78 & 0.77 & 0.79 & 0.78 & 0.73 & 0.73 \\ 
   & 5 & 0.05 & 0.81 & 0.82 & 0.74 & 0.71 & 0.82 & 0.82 & 0.70 & 0.70 \\ 
   &  & 0.25 & 0.83 & 0.84 & 0.76 & 0.75 & 0.83 & 0.84 & 0.71 & 0.71 \\ 
   & 10 & 0.05 & 0.81 & 0.81 & 0.50 & 0.55 & 0.83 & 0.79 & 0.70 & 0.69 \\ 
   &  & 0.25 & 0.83 & 0.82 & 0.49 & 0.59 & 0.82 & 0.83 & 0.71 & 0.71 \\ 
   \hline
500 & 2 & 0.05 & 0.77 & 0.77 & 0.76 & 0.76 & 0.78 & 0.76 & 0.73 & 0.72 \\ 
   &  & 0.25 & 0.78 & 0.78 & 0.78 & 0.77 & 0.78 & 0.77 & 0.74 & 0.73 \\ 
   & 5 & 0.05 & 0.82 & 0.81 & 0.77 & 0.72 & 0.82 & 0.81 & 0.70 & 0.70 \\ 
   &  & 0.25 & 0.83 & 0.83 & 0.79 & 0.74 & 0.82 & 0.83 & 0.71 & 0.71 \\ 
   & 10 & 0.05 & 0.80 & 0.81 & 0.45 & 0.64 & 0.82 & 0.80 & 0.70 & 0.70 \\ 
   &  & 0.25 & 0.82 & 0.82 & 0.57 & 0.59 & 0.83 & 0.81 & 0.71 & 0.71 \\ 
   \hline
\end{tabular}
}
\end{table}
When the true generative model is linear, logistic regression and 
the linear SVM outperform the other methods in terms of 
sensitivity and specificity. The Gaussian SVM performs poorly 
in the presence of noise variables. 

Table~\ref{sim1.sesp.lin.uw} 
contains estimated sensitivities and specificities 
of an unweighted SVM when the true model is linear.
%\input{roc.sim1.sesp.lin.uw.tables.txt}
% latex table generated in R 3.2.2 by xtable 1.7-4 package
% Tue Apr 18 17:50:09 2017
\begin{table}[h!]
\caption{Average sensitivity and specificity of unweighted SVM when true model is linear.} 
\label{sim1.sesp.lin.uw}
\centering
\begin{tabular}{ccc|cccc}
  \hline
   & & & \multicolumn{2}{c}{Linear SVM} & \multicolumn{2}{c}{Gaussian SVM} \\$n$ & $p$ & $q$ & $se$ & $sp$ & $se$ & $sp$ \\ 
   \hline
250 & 2 & 0.05 & 0.60 & 0.86 & 0.57 & 0.85 \\ 
   &  & 0.25 & 0.70 & 0.81 & 0.68 & 0.79 \\ 
   & 5 & 0.05 & 0.62 & 0.89 & 0.37 & 0.88 \\ 
   &  & 0.25 & 0.74 & 0.86 & 0.54 & 0.83 \\ 
   & 10 & 0.05 & 0.62 & 0.88 & 0.00 & 1.00 \\ 
   &  & 0.25 & 0.73 & 0.85 & 0.02 & 0.99 \\ 
   \hline
500 & 2 & 0.05 & 0.63 & 0.84 & 0.61 & 0.84 \\ 
   &  & 0.25 & 0.72 & 0.80 & 0.71 & 0.79 \\ 
   & 5 & 0.05 & 0.63 & 0.90 & 0.45 & 0.87 \\ 
   &  & 0.25 & 0.74 & 0.87 & 0.61 & 0.82 \\ 
   & 10 & 0.05 & 0.62 & 0.90 & 0.00 & 1.00 \\ 
   &  & 0.25 & 0.73 & 0.86 & 0.02 & 0.99 \\ 
   \hline
\end{tabular}
\end{table}

%\nocite{*}
\bibliographystyle{Chicago}
\bibliography{roc}

\begin{thebibliography}{}

\bibitem[\protect\citeauthoryear{Bartlett, Jordan, and McAuliffe}{Bartlett
  et~al.}{2006}]{bartlett2006convexity}
Bartlett, P.~L., M.~I. Jordan, and J.~D. McAuliffe (2006).
\newblock Convexity, classification, and risk bounds.
\newblock {\em Journal of the American Statistical Association\/}~{\em
  101\/}(473), 138--156.

\bibitem[\protect\citeauthoryear{Bradley}{Bradley}{1997}]{bradley1997use}
Bradley, A.~P. (1997).
\newblock The use of the area under the roc curve in the evaluation of machine
  learning algorithms.
\newblock {\em Pattern Recognition\/}~{\em 30\/}(7), 1145--1159.

\bibitem[\protect\citeauthoryear{Breiman}{Breiman}{2001}]{breiman2001random}
Breiman, L. (2001).
\newblock Random forests.
\newblock {\em Machine Learning\/}~{\em 45\/}(1), 5--32.

\bibitem[\protect\citeauthoryear{Cai and Dodd}{Cai and
  Dodd}{2008}]{cai2008regression}
Cai, T. and L.~E. Dodd (2008).
\newblock Regression analysis for the partial area under the {ROC} curve.
\newblock {\em Statistica Sinica\/}~{\em 18\/}(3), 817--836.

\bibitem[\protect\citeauthoryear{Cai and Moskowitz}{Cai and
  Moskowitz}{2004}]{cai2004semi}
Cai, T. and C.~S. Moskowitz (2004).
\newblock Semi-parametric estimation of the binormal {ROC} curve for a
  continuous diagnostic test.
\newblock {\em Biostatistics\/}~{\em 5\/}(4), 573--586.

\bibitem[\protect\citeauthoryear{Cai and Pepe}{Cai and
  Pepe}{2002}]{cai2002semiparametric}
Cai, T. and M.~S. Pepe (2002).
\newblock Semiparametric receiver operating characteristic analysis to evaluate
  biomarkers for disease.
\newblock {\em Journal of the American Statistical Association\/}~{\em
  97\/}(460), 1099--1107.

\bibitem[\protect\citeauthoryear{Campbell}{Campbell}{1994}]{campbell1994advances}
Campbell, G. (1994).
\newblock Advances in statistical methodology for the evaluation of diagnostic
  and laboratory tests.
\newblock {\em Statistics in Medicine\/}~{\em 13\/}(5-7), 499--508.

\bibitem[\protect\citeauthoryear{Chang and Lin}{Chang and
  Lin}{2011}]{chang2011libsvm}
Chang, C.-C. and C.-J. Lin (2011).
\newblock {LIBSVM}: A library for support vector machines.
\newblock {\em ACM Transactions on Intelligent Systems and Technology\/}~{\em
  2}, 27:1--27:27.
\newblock Software available at \url{http://www.csie.ntu.edu.tw/~cjlin/libsvm}.

\bibitem[\protect\citeauthoryear{Claeskens, Jing, Peng, and Zhou}{Claeskens
  et~al.}{2003}]{claeskens2003empirical}
Claeskens, G., B.-Y. Jing, L.~Peng, and W.~Zhou (2003).
\newblock Empirical likelihood confidence regions for comparison distributions
  and roc curves.
\newblock {\em Canadian Journal of Statistics\/}~{\em 31\/}(2), 173--190.

\bibitem[\protect\citeauthoryear{Cortes and Vapnik}{Cortes and
  Vapnik}{1995}]{cortes1995support}
Cortes, C. and V.~Vapnik (1995).
\newblock Support-vector networks.
\newblock {\em Machine Learning\/}~{\em 20\/}(3), 273--297.

\bibitem[\protect\citeauthoryear{Dasgupta, Goldberg, and Kosorok}{Dasgupta
  et~al.}{2015}]{dasgupta2015feature}
Dasgupta, S., Y.~Goldberg, and M.~R. Kosorok (2015).
\newblock Feature elimination in support vector machines and empirical risk
  minimization.
\newblock {\em The University of North Carolina at Chapel Hill Department of
  Biostatistics Technical Report Series\/}.
\newblock Working paper 44.

\bibitem[\protect\citeauthoryear{Duda, Hart, and Stork}{Duda
  et~al.}{2012}]{duda2012pattern}
Duda, R.~O., P.~E. Hart, and D.~G. Stork (2012).
\newblock {\em Pattern classification}.
\newblock John Wiley \& Sons.

\bibitem[\protect\citeauthoryear{Etzioni, Pepe, Longton, Hu, and
  Goodman}{Etzioni et~al.}{1999}]{etzioni1999incorporating}
Etzioni, R., M.~Pepe, G.~Longton, C.~Hu, and G.~Goodman (1999).
\newblock Incorporating the time dimension in receiver operating characteristic
  curves: A case study of prostate cancer.
\newblock {\em Medical Decision Making\/}~{\em 19\/}(3), 242--251.

\bibitem[\protect\citeauthoryear{Fan, Prat, Parker, Liu, Carey, Troester, and
  Perou}{Fan et~al.}{2011}]{fan2011building}
Fan, C., A.~Prat, J.~S. Parker, Y.~Liu, L.~A. Carey, M.~A. Troester, and C.~M.
  Perou (2011).
\newblock Building prognostic models for breast cancer patients using clinical
  variables and hundreds of gene expression signatures.
\newblock {\em BMC Medical Genomics\/}~{\em 4\/}(1), 3.

\bibitem[\protect\citeauthoryear{Hastie, Tibshirani, and Friedman}{Hastie
  et~al.}{2009}]{friedman2001elements}
Hastie, T., R.~Tibshirani, and J.~Friedman (2009).
\newblock {\em The Elements of Statistical Learning\/} (2 ed.).
\newblock New York: Springer-Verlag.

\bibitem[\protect\citeauthoryear{Horv{\'a}th, Horv{\'a}th, and
  Zhou}{Horv{\'a}th et~al.}{2008}]{horvath2008confidence}
Horv{\'a}th, L., Z.~Horv{\'a}th, and W.~Zhou (2008).
\newblock Confidence bands for {ROC} curves.
\newblock {\em Journal of Statistical Planning and Inference\/}~{\em 138\/}(6),
  1894--1904.

\bibitem[\protect\citeauthoryear{Jensen, M{\"u}ller, and Sch{\"a}fer}{Jensen
  et~al.}{2000}]{jensen2000regional}
Jensen, K., H.-H. M{\"u}ller, and H.~Sch{\"a}fer (2000).
\newblock Regional confidence bands for {ROC} curves.
\newblock {\em Statistics in Medicine\/}~{\em 19\/}(4), 493--509.

\bibitem[\protect\citeauthoryear{Kosorok}{Kosorok}{2008}]{kosorok2008introduction}
Kosorok, M.~R. (2008).
\newblock {\em Introduction to Empirical Processes and Semiparametric
  Inference}.
\newblock New York: Springer Science \& Business Media.

\bibitem[\protect\citeauthoryear{Krzy{\.z}ak, Linder, and Lugosi}{Krzy{\.z}ak
  et~al.}{1996}]{krzyzak1996nonparametric}
Krzy{\.z}ak, A., T.~Linder, and G.~Lugosi (1996).
\newblock Nonparametric estimation and classification using radial basis
  function nets and empirical risk minimization.
\newblock {\em {IEEE} Transactions on Neural Networks\/}~{\em 7\/}(2),
  475--487.

\bibitem[\protect\citeauthoryear{Lin, Lin, and Weng}{Lin
  et~al.}{2007}]{lin2007note}
Lin, H.-T., C.-J. Lin, and R.~C. Weng (2007).
\newblock A note on {Platt’s} probabilistic outputs for support vector
  machines.
\newblock {\em Machine Mearning\/}~{\em 68\/}(3), 267--276.

\bibitem[\protect\citeauthoryear{Lin}{Lin}{2002}]{lin2002support}
Lin, Y. (2002).
\newblock Support vector machines and the {B}ayes rule in classification.
\newblock {\em Data Mining and Knowledge Discovery\/}~{\em 6\/}(3), 259--275.

\bibitem[\protect\citeauthoryear{L{\'o}pez-Rat{\'o}n,
  Rodr{\'\i}guez-{\'A}lvarez, Cadarso-Su{\'a}rez, and
  Gude-Sampedro}{L{\'o}pez-Rat{\'o}n et~al.}{2014}]{lopez2014optimalcutpoints}
L{\'o}pez-Rat{\'o}n, M., M.~X. Rodr{\'\i}guez-{\'A}lvarez,
  C.~Cadarso-Su{\'a}rez, and F.~Gude-Sampedro (2014).
\newblock Optimalcutpoints: An {R} package for selecting optimal cutpoints in
  diagnostic tests.
\newblock {\em Journal of Statistical Software\/}~{\em 61}, 1--36.

\bibitem[\protect\citeauthoryear{Ma and Hall}{Ma and
  Hall}{1993}]{ma1993confidence}
Ma, G. and W.~Hall (1993).
\newblock Confidence bands for receiver operating characteristic curves.
\newblock {\em Medical Decision Making\/}~{\em 13\/}(3), 191--197.

\bibitem[\protect\citeauthoryear{McIntosh and Pepe}{McIntosh and
  Pepe}{2002}]{mcintosh2002combining}
McIntosh, M.~W. and M.~S. Pepe (2002).
\newblock Combining several screening tests: Optimality of the risk score.
\newblock {\em Biometrics\/}~{\em 58\/}(3), 657--664.

\bibitem[\protect\citeauthoryear{Pepe}{Pepe}{1997}]{pepe1997regression}
Pepe, M.~S. (1997).
\newblock A regression modelling framework for receiver operating
  characteristic curves in medical diagnostic testing.
\newblock {\em Biometrika\/}~{\em 84\/}(3), 595--608.

\bibitem[\protect\citeauthoryear{Pepe}{Pepe}{2000}]{pepe2000interpretation}
Pepe, M.~S. (2000).
\newblock An interpretation for the {ROC} curve and inference using {GLM}
  procedures.
\newblock {\em Biometrics\/}~{\em 56\/}(2), 352--359.

\bibitem[\protect\citeauthoryear{Pepe}{Pepe}{2003}]{pepe2003statistical}
Pepe, M.~S. (2003).
\newblock {\em The Statistical Evaluation of Medical Tests for Classification
  and Prediction}.
\newblock Oxford: Oxford University Press.

\bibitem[\protect\citeauthoryear{Platt}{Platt}{1999}]{platt1999probabilistic}
Platt, J. (1999).
\newblock Probabilistic outputs for support vector machines and comparisons to
  regularized likelihood methods.
\newblock {\em Advances in Large Margin Classifiers\/}~{\em 10\/}(3), 61--74.

\bibitem[\protect\citeauthoryear{Provost and Fawcett}{Provost and
  Fawcett}{1998}]{provost1998robust}
Provost, F. and T.~Fawcett (1998).
\newblock Robust classification systems for imprecise environments.
\newblock In {\em Proceedings of the Fifteenth National Conference on
  Artificial Intelligence}, pp.\  706--713.

\bibitem[\protect\citeauthoryear{Provost and Fawcett}{Provost and
  Fawcett}{1997}]{provost1997analysis}
Provost, F.~J. and T.~Fawcett (1997).
\newblock Analysis and visualization of classifier performance: comparison
  under imprecise class and cost distributions.
\newblock In {\em Proceedings of the Third International Conference on
  Knowledge Discovery and Data Mining}, pp.\  43--48.

\bibitem[\protect\citeauthoryear{{R Core Team}}{{R Core
  Team}}{2016}]{rcoreteam}
{R Core Team} (2016).
\newblock {\em R: A Language and Environment for Statistical Computing}.
\newblock Vienna, Austria: R Foundation for Statistical Computing.

\bibitem[\protect\citeauthoryear{Shebl, El-Kamary, Saleh, Abdel-Hamid, Mikhail,
  Allam, El-Arabi, Elhenawy, El-Kafrawy, El-Daly, et~al.}{Shebl
  et~al.}{2009}]{shebl2009prospective}
Shebl, F.~M., S.~S. El-Kamary, D.~A. Saleh, M.~Abdel-Hamid, N.~Mikhail,
  A.~Allam, H.~El-Arabi, I.~Elhenawy, S.~El-Kafrawy, M.~El-Daly, et~al. (2009).
\newblock Prospective cohort study of mother-to-infant infection and clearance
  of hepatitis {C} in rural {E}gyptian villages.
\newblock {\em Journal of Medical Virology\/}~{\em 81\/}(6), 1024--1031.

\bibitem[\protect\citeauthoryear{Spackman}{Spackman}{1989}]{spackman1989signal}
Spackman, K.~A. (1989).
\newblock Signal detection theory: Valuable tools for evaluating inductive
  learning.
\newblock In {\em Proceedings of the Sixth International Workshop on Machine
  Learning}, pp.\  160--163. Elsevier.

\bibitem[\protect\citeauthoryear{Steinwart and Christmann}{Steinwart and
  Christmann}{2008}]{steinwart2008support}
Steinwart, I. and A.~Christmann (2008).
\newblock {\em Support Vector Machines}.
\newblock New York: Springer Science \& Business Media.

\bibitem[\protect\citeauthoryear{Tibshirani}{Tibshirani}{1996}]{tibshirani1996regression}
Tibshirani, R. (1996).
\newblock Regression shrinkage and selection via the {LASSO}.
\newblock {\em Journal of the Royal Statistical Society, Series B\/}, 267--288.

\bibitem[\protect\citeauthoryear{Vapnik}{Vapnik}{1998}]{vapnik1998statistical}
Vapnik, V. (1998).
\newblock {\em Statistical Learning Theory. 1998}.
\newblock Wiley, New York.

\bibitem[\protect\citeauthoryear{Veropoulos, Campbell, Cristianini,
  et~al.}{Veropoulos et~al.}{1999}]{veropoulos1999controlling}
Veropoulos, K., C.~Campbell, N.~Cristianini, et~al. (1999).
\newblock Controlling the sensitivity of support vector machines.
\newblock In {\em Proceedings of the International Joint Conference on {AI}},
  pp.\  55--60.

\bibitem[\protect\citeauthoryear{Zhang}{Zhang}{2004}]{zhang2004statistical}
Zhang, T. (2004).
\newblock Statistical behavior and consistency of classification methods based
  on convex risk minimization.
\newblock {\em Annals of Statistics\/}~{\em 32\/}(1), 56--85.

\bibitem[\protect\citeauthoryear{Zhao, Zeng, Rush, and Kosorok}{Zhao
  et~al.}{2012}]{zhao2012estimating}
Zhao, Y., D.~Zeng, A.~J. Rush, and M.~R. Kosorok (2012).
\newblock Estimating individualized treatment rules using outcome weighted
  learning.
\newblock {\em Journal of the American Statistical Association\/}~{\em
  107\/}(499), 1106--1118.

\bibitem[\protect\citeauthoryear{Zhou, McClish, and Obuchowski}{Zhou
  et~al.}{2002}]{zhou2002statistical}
Zhou, X.-H., D.~K. McClish, and N.~A. Obuchowski (2002).
\newblock {\em Statistical Methods in Diagnostic Medicine}.
\newblock New York: John Wiley \& Sons.

\end{thebibliography}

\end{document}